\documentclass[review,twocolumn]{autart}
\usepackage[T1]{fontenc}
\usepackage{booktabs}
\usepackage{colortbl}
\usepackage{graphicx,color} 
\usepackage{epsfig} 
\usepackage{times} 
\usepackage{amsmath} 
\usepackage{amssymb}  
\usepackage{caption}
\usepackage{multirow,tabularx}
\usepackage{subfigure}[!t]
\usepackage{bbm}
\usepackage{ntheorem}

\usepackage{url}
\usepackage{multirow}
\usepackage{nameref}

\usepackage{algorithm}
\usepackage{pgf}
\usepackage{pgfplots}
\usepgfplotslibrary{fillbetween}
\usepackage{tikz}
\usetikzlibrary{arrows,automata}
\usetikzlibrary{fit}
\usetikzlibrary{intersections}
\usetikzlibrary{calc}
\usetikzlibrary{backgrounds}
\usetikzlibrary{shapes}
\usetikzlibrary{arrows}
\usetikzlibrary{arrows.meta,quotes}
\usetikzlibrary{positioning}
\usetikzlibrary{decorations.text}
\pgfdeclarelayer{background}
\pgfdeclarelayer{foreground}
\tikzstyle{arrow} = [thick,->,>=stealth]
\tikzstyle{process} = [rectangle, minimum width=2.5cm, minimum height=1cm, line width=0.8pt, text centered, draw=black, fill=gray!10]
\tikzstyle{sum} = [circle, minimum width=0.5cm, minimum height=0.5cm, line width=0.8pt, draw=black,  fill=gray!30]
\pgfsetlayers{background,main,foreground}
\usetikzlibrary{calc,patterns,decorations.pathmorphing,decorations.markings}

\graphicspath{{figures/}}

\usepackage{enumitem}
\renewcommand{\theenumi}{\arabic{enumi}}

\renewcommand{\theenumii}{\arabic{enumii}}

\renewcommand{\theenumiii}{\arabic{enumiii}}

\newtheorem{lemma}{Lemma}
\newtheorem{theorem}{Theorem}
\newtheorem{definition}{Definition}
\newtheorem{assumption}{Assumption}
\newtheorem{proposition}{Proposition}
\newtheorem{remark}{Remark}
\newtheorem*{proof}{Proof}

\newcommand{\VB}[1]{\textcolor{black}{#1}} 

\begin{document}
	\tikzstyle{decision} = [diamond, draw, fill=blue!5, text badly centered, inner sep=0pt]
	\tikzstyle{block} = [rectangle, draw, fill=blue!5, 
	text width=15em, text centered, rounded corners, minimum height=4em]
	\tikzstyle{line} = [draw, -latex']
	\tikzstyle{start} = [draw, ellipse,fill=red!5,minimum height=2em]

	\begin{frontmatter}
		\title{Explainable data-driven modeling via mixture of experts:\\ towards effective blending of grey and black-box models} 

		\author[POLIMI]{Jessica Leoni}\ead{jessica.leoni@polimi.it},   	
		\author[TUE]{Valentina Breschi}, 	
		\author[POLIMI]{Simone Formentin}, 
		\author[POLIMI]{Mara Tanelli}
		\address[POLIMI]{Dipartimento di Elettronica, Bioingegneria e Informazione, Politecnico di Milano, Via Ponzio 34/5, Milano, Italy}  
		\address[TUE]{Department of Electrical Engineering, Eindhoven University of Technology, 5600 MB Eindhoven, The Netherlands}  %
		\begin{abstract}
                Traditional models grounded in first principles often struggle with accuracy as the system's complexity increases. Conversely, machine learning approaches, while powerful, face challenges in interpretability and in handling physical constraints. Efforts to combine these models often often stumble upon difficulties in finding a balance between accuracy and complexity. To address these issues, we propose a comprehensive framework based on a \textquotedblleft mixture of experts\textquotedblright \ rationale. This approach enables the data-based fusion of diverse local models, leveraging the full potential of first-principle-based priors. Our solution allows independent training of experts, drawing on techniques from both machine learning and system identification, and it supports both collaborative and competitive learning paradigms. To enhance interpretability, we 
                penalize abrupt variations in the expert's combination. Experimental results validate the effectiveness of our approach in producing an interpretable combination of models closely resembling the target phenomena. 

		\end{abstract}
		
	\end{frontmatter}
	
	\section{Introduction}\label{Sec:intro}
 Over recent decades, advances in mechanics and electronics have led to the development of increasingly sophisticated systems with complex and multi-physics dynamics, exposing limitations in first principle-based representations \cite{fujimoto2017research}. Modeling these advanced systems purely based on domain knowledge may inadequately capture the overall system behavior, often necessitating the formulation of complex partial differential equations. However, this approach may compromise model interpretability and be computationally prohibitive in practical applications \cite{karniadakis2021physics}.

These challenges have driven progress in system identification and machine learning, providing data-driven alternatives to traditional physics-based approaches. Machine learning (ML), in particular, has demonstrated its effectiveness in capturing intricate patterns in data and in making accurate predictions about system behavior without explicit knowledge of underlying physics \cite{bikmukhametov2020combining}. Despite its success, ML models often lack interpretability and face limitations in incorporating the physical constraints of a system \cite{xiong2002grey}.


        \begin{figure*}
        \centering
        \begin{tabular}{ccc}
        \subfigure[Serial: training]{
        \scalebox{.5}{\begin{tikzpicture}
            \node[coordinate] (trw) {}; 
                \node[draw,rectangle, below right of=trw,node distance=1.25cm,minimum width=1em,minimum height=2.5em,fill=gray!5!white] (whitebox) {Gray box};
			\node[coordinate,below of=trw,node distance=2.5cm]         (tlw) {}; 
                \draw[->] (trw) -| node[xshift=-1.75cm]{\begin{tabular}{c}Training\vspace{-.2cm}\\regressors\end{tabular}}(whitebox);
                \draw[->] (tlw) -| node[xshift=-1.75cm]{\begin{tabular}{c}Training\vspace{-.2cm}\\outputs\end{tabular}}(whitebox);
                 \node[coordinate, right of=whitebox,node distance=.85cm] (aid1) {};
                \node[coordinate, below of=aid1,node distance=.3cm] (aid1bis) {};
                \node[draw,circle,right of=tlw,node distance=2.25cm] (sum1) {};
                \draw[->] (tlw) -- node[xshift=.8cm,yshift=-.2cm]{-}(sum1);
                \draw[->,blue] (aid1bis) -| node[xshift=1.2cm,yshift=-.3cm] {\begin{tabular}{c}\textcolor{blue}{Reconstructed}\vspace{-.2cm}\\ \textcolor{blue}{outputs}\end{tabular}} node[xshift=-.2cm,yshift=-1cm]{\textcolor{black}{+}}(sum1);
                \node[coordinate,right of=whitebox,node distance=.85cm] (aid2) {};
                \node[coordinate,above of=aid2,node distance=.3cm] (aid2bis) {};
                \node[coordinate,right of=aid2bis,node distance=1.25cm] (outw1) {};
                \draw[->] (aid2bis) -- node[xshift=1.45cm]{\begin{tabular}{c}Gray box\vspace{-.2cm}\\Parameters\end{tabular}}(outw1);
                \node[coordinate,below of=sum1,node distance=1cm] (trb) {}; 
                \node[draw,rectangle, fill=black, right of=sum1,node distance=3cm,minimum width=1em,minimum height=2.5em] (blackbox) {\textcolor{white}{Black box}};
                \draw[->] (sum1) -- (blackbox);
                \draw[->] (trb) -| node[xshift=-3.75cm]{\begin{tabular}{c}Training\vspace{-.2cm}\\regressors\end{tabular}}(blackbox);
               \node[coordinate,right of=blackbox,node distance=1.5cm] (outb) {};
                \draw[->] (blackbox) -- node[xshift=1.25cm]{\begin{tabular}{c}Black box\vspace{-.2cm}\\Parameters\end{tabular}}(outb);
        \end{tikzpicture}}} & \subfigure[Parallel: training]{\scalebox{.5}{\begin{tikzpicture}
                \node[coordinate] (trw) {}; 
                \node[draw,rectangle, below right of=trw,node distance=1.25cm,minimum width=1em,minimum height=2.5em,fill=gray!5!white] (whitebox) {Gray box};
			\node[coordinate,below of=trw,node distance=1.75cm]         (tlw) {}; 
                \draw[->] (trw) -| node[xshift=-1.75cm]{\begin{tabular}{c}Training\vspace{-.2cm}\\regressors\end{tabular}}(whitebox);
                \draw[->] (tlw) -| node[xshift=-1.75cm]{\begin{tabular}{c}Training\vspace{-.2cm}\\outputs\end{tabular}}(whitebox);
                \node[coordinate,right of=whitebox,node distance=1.5cm] (outw) {};
                \draw[->] (whitebox) -- node[xshift=1.25cm]{\begin{tabular}{c}Gray box\vspace{-.2cm}\\Parameters\end{tabular}}(outw);
                \node[coordinate,below of=trw,node distance=3cm] (trb) {}; 
                \node[draw,rectangle, fill=black, below right of=trb,node distance=1.25cm,minimum width=1em,minimum height=2.5em] (blackbox) {\textcolor{white}{Black box}};
                \node[coordinate,below of=trb,node distance=1.75cm]         (tlb) {}; 
                \draw[->] (trb) -| node[xshift=-1.75cm]{\begin{tabular}{c}Training\vspace{-.2cm}\\regressors\end{tabular}}(blackbox);
                \draw[->] (tlb) -| node[xshift=-1.75cm]{\begin{tabular}{c}Training\vspace{-.2cm}\\outputs\end{tabular}}(blackbox);
                \node[coordinate,right of=blackbox,node distance=1.5cm] (outb) {};
                \draw[->] (blackbox) -- node[xshift=1.25cm]{\begin{tabular}{c}Black box\vspace{-.2cm}\\Parameters\end{tabular}}(outb);
        \end{tikzpicture}}} & \subfigure[Ensemble: training]{\scalebox{.5}{\begin{tikzpicture}
                \node[coordinate] (trw) {}; 
                \node[draw,rectangle, below right of=trw,node distance=1.25cm,minimum width=1em,minimum height=2.5em, fill=black] (blackbox1) {{\textcolor{white}{Black box}}};
			\node[coordinate,below of=trw,node distance=1.75cm]         (tlw) {}; 
                \draw[->] (trw) -| node[xshift=-2cm]{\begin{tabular}{c}Partition of\vspace{-.2cm}\\training regressors\end{tabular}}(blackbox1);
                \draw[->] (tlw) -| node[xshift=-2cm]{\begin{tabular}{c}Partition of\vspace{-.2cm}\\training outputs\end{tabular}}(blackbox1);
                \node[coordinate,right of=blackbox1,node distance=1.5cm] (outw) {};
                \draw[->] (blackbox1) -- node[xshift=1.25cm]{\begin{tabular}{c}Black box\vspace{-.2cm}\\Parameters\end{tabular}}(outw);
                \node[coordinate,below of=trw,node distance=3cm] (trc) {}; 
                \node[draw,rectangle, fill=black, below right of=trc,node distance=1.25cm,minimum width=1em,minimum height=2.5em] (blackbox2) {\textcolor{white}{Black box}};
                \node[coordinate,below of=trc,node distance=1.75cm]         (tlc) {}; 
                \draw[->] (trc) -| node[xshift=-2cm]{\begin{tabular}{c}Partition of\vspace{-.2cm}\\training regressors\end{tabular}}(blackbox2);
                \draw[->] (tlc) -| node[xshift=-2cm]{\begin{tabular}{c}Partition of\vspace{-.2cm}\\training outputs\end{tabular}}(blackbox2);
                \node[coordinate,right of=blackbox2,node distance=1.5cm] (outc) {};
                \draw[->] (blackbox2) -- node[xshift=1.25cm]{\begin{tabular}{c}Black box\vspace{-.2cm}\\Parameters\end{tabular}}(outc);
        \end{tikzpicture}}}\\
        \subfigure[Serial: prediction]{
        \scalebox{.5}{\begin{tikzpicture}
            \node[rectangle] (data2) {\begin{tabular}{c}New \vspace{-.2cm}\\regressors\end{tabular}};
             \node[coordinate,right of=data2,node distance=1cm] (data) {};
             \node[coordinate,above of=data2,node distance=.5cm] (input1) {};
             \node[coordinate,below of=data2,node distance=.5cm] (input2) {};
             \node[draw,rectangle, above right of=data,node distance=1.25cm,minimum width=1em,minimum height=2.5em,fill=gray!5!white] (whitebox) {Gray box};
             \node[draw,rectangle, fill=black, below right of=data,node distance=1.25cm,minimum width=1em,minimum height=2.5em] (blackbox) {\textcolor{white}{Black box}};
             \node[draw,circle,right of=data=node,node distance=2.5cm] (sum1) {}; 
             \node[coordinate,right of=sum1,node distance=.75cm] (outp) {};
             \draw[->] (input1) |- node[]{}(whitebox);
             \draw[->] (input2)  |- node[]{}(blackbox);
             \draw[->] (whitebox) -| node[xshift=-.15cm,yshift=-.5cm]{+} (sum1);
             \draw[->] (blackbox) -| node[xshift=-.15cm,yshift=.5cm]{+} (sum1);
             \draw[->] (sum1) -- node[xshift=1.15cm]{Prediction} (outp);
             \node[rectangle,below of=blackbox,node distance=.5cm] (aid3) {}; 
        \end{tikzpicture}}} & \subfigure[Parallel: prediction]{\scalebox{.5}{\begin{tikzpicture}
            \node[rectangle] (data2) {\begin{tabular}{c}New \vspace{-.2cm}\\regressors\end{tabular}};
             \node[coordinate,right of=data2,node distance=1cm] (data) {};
             \node[coordinate,above of=data2,node distance=.5cm] (input1) {};
             \node[coordinate,below of=data2,node distance=.5cm] (input2) {};
             \node[draw,rectangle, above right of=data,node distance=1.25cm,minimum width=1em,minimum height=2.5em,fill=gray!5!white] (whitebox) {Gray box};
             \node[draw,rectangle, fill=black, below right of=data,node distance=1.25cm,minimum width=1em,minimum height=2.5em] (blackbox) {\textcolor{white}{Black box}};
             \node[draw,circle,right of=data=node,node distance=2.5cm] (sum1) {}; 
             \node[coordinate,right of=sum1,node distance=.75cm] (outp) {};
             \draw[->] (input1) |- node[]{}(whitebox);
             \draw[->] (input2)  |- node[]{}(blackbox);
             \draw[->] (whitebox) -| node[xshift=-.15cm,yshift=-.5cm]{+} (sum1);
             \draw[->] (blackbox) -| node[xshift=-.15cm,yshift=.5cm]{+} (sum1);
             \draw[->] (sum1) -- node[xshift=1.15cm]{Prediction} (outp);
             \node[rectangle,below of=blackbox,node distance=.5cm] (aid3) {};
        \end{tikzpicture}}} & 
        \subfigure[Ensemble: prediction]{
        \scalebox{.5}{\begin{tikzpicture}
        \node[coordinate] (start) {}; 
         \node[draw, fill=blue!5!white, rectangle,minimum width=1em,minimum height=2.5em, right of=trw, node distance=1.5cm] (splitter) {\begin{tabular}{c}Data\vspace{-.2cm}\\ split\end{tabular}};
         \node[coordinate,above of=splitter,node distance=1cm] (aid1){};
         \node[draw, rectangle, right of=aid1, node distance=2.5cm, minimum width=1em,minimum height=2.5em, fill=black] (model1) {{\textcolor{white}{Black box}}};
         \node[coordinate,below of=splitter,node distance=1cm] (aid2){};
            \node[draw, rectangle, right of=aid2, node distance=2.5cm, minimum width=1em,minimum height=2.5em, fill=black] (model2) {{\textcolor{white}{Black box}}};
        \node[draw, fill=blue!5!white, rectangle, minimum width=1em,minimum height=2.5em, right of=splitter, node distance=5cm] (merge) {\begin{tabular}{c}Output\vspace{-.2cm}\\ combiner\end{tabular}};
         \node[rectangle,below of=model2,node distance=.5cm] (aid3) {};
        \node[coordinate, right of=merge, node distance=1.5cm] (end) {};
        \draw[->] (splitter) |- (model1);
        \draw[->] (splitter) |- (model2);
        \draw[->] (start) -- node[xshift=-1cm]{\begin{tabular}{c}New\vspace{-.2cm}\\regressors\end{tabular}}(splitter);
        \draw[->] (model1) -| (merge);
        \draw[->] (model2) -| (merge);
        \draw[->] (merge) -- node[xshift=1.2cm]{Predictions} (end);
        \end{tikzpicture}}}  
        \end{tabular}
        \caption{A schematic overview of the training [upper panels] and prediction [lower panels] logic of serial, parallel, and ensemble approaches. Gray and black box models are respectively depicted with gray and black rectangles, respectively. For the sake of simplicity, the ensemble is depicted considering only 2 black-box local models, but it can in principle include as many experts as one likes. }\label{fig:schemes_1}
    \end{figure*}
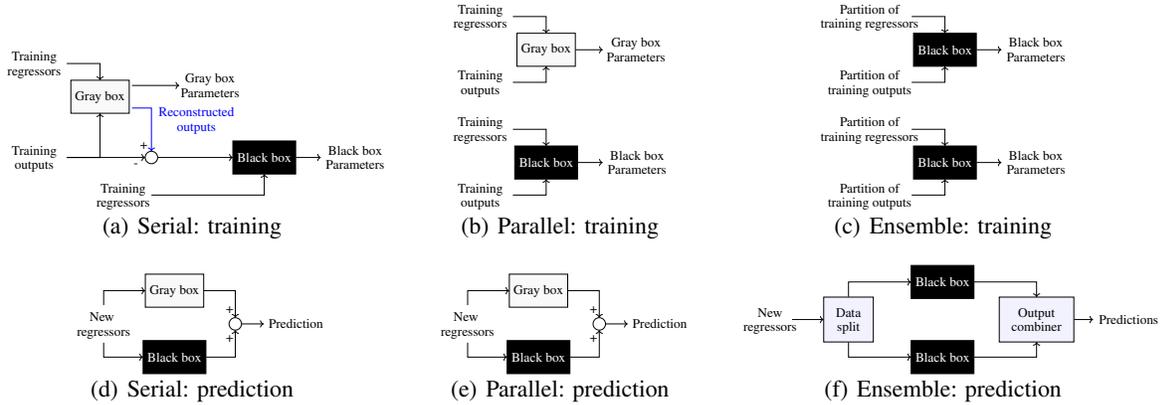

Efforts to integrate the interpretability and robustness of first principles with the predictive capabilities of machine learning (ML) have resulted in various approaches \cite{molnar2020interpretable, du2019techniques}. These approaches fall into four categories: \emph{physic-constrained}, \emph{serial}, \emph{parallel}, and \emph{ensemble} strategies.
In the \emph{physic-constrained} category, techniques either integrate physically meaningful features from first principles into ML models or explicitly include physical constraints, such as boundary conditions, into the loss function (see, e.g., the working principle of physics-informed neural networks (PINN)) \cite{bikmukhametov2020combining, cuomo2022scientific}. Despite their successful application, these techniques often yield models that are still black boxes, limiting interpretability and explainability.
\emph{Serial approaches}, illustrated in \figurename{~\ref{fig:schemes_1}}, use data to fit a simple model and ML models to capture the residual, representing unexplained aspects by the simple physics-based model \cite{gray2018hybrid}. Both \emph{gray box} physics-based and \emph{black box} ML models contribute to predictions, but the non-negligible noise contribution to the residual can impact the quality of the ML model.
Instead, \emph{parallel approaches} involve concurrently fitting both gray box and black box models on the same data for predictions \cite{xiong2002grey}. Although aiming for a more comprehensive system description, this structure may not provide a clear understanding of the standalone physics-based and ML models' range of validity.
\emph{Ensemble strategies}, such as bagging, boosting, and mixtures of experts (MoE), operate based on the \emph{divide et impera} principle, where different models learn specific patterns and combine predictions to enhance accuracy \cite{breiman1996bagging, freund1996experiments, jordan1994hierarchical}. MoEs, in particular, assign specialized ML models to different feature spaces and combine their outputs through gating, providing an explanation of each submodel's role. However, existing MoE architectures often lack interpretability as they primarily rely on neural networks and seldom include physics-based experts \cite{bischof2022mixture}.

    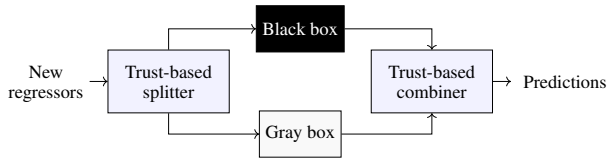
\begin{figure}[!tb]
    \centering
         \scalebox{.7}{\begin{tikzpicture}
        \node[coordinate] (start) {}; 
         \node[draw, fill=blue!5!white, rectangle,minimum width=1em,minimum height=2.5em, right of=trw, node distance=1.5cm] (splitter) {\begin{tabular}{c}Trust-based\vspace{-.2cm}\\ splitter\end{tabular}};
         \node[coordinate,above of=splitter,node distance=1cm] (aid1){};
         \node[draw, rectangle, right of=aid1, node distance=2.5cm, minimum width=1em,minimum height=2.5em, fill=black] (model1) {{\textcolor{white}{Black box}}};
         \node[coordinate,below of=splitter,node distance=1cm] (aid2){};
            \node[draw, rectangle, right of=aid2, fill=gray!5!white, node distance=2.5cm, minimum width=1em,minimum height=2.5em] (model2) {{Gray box}};
        \node[draw, fill=blue!5!white, rectangle, minimum width=1em,minimum height=2.5em, right of=splitter, node distance=5cm] (merge) {\begin{tabular}{c}Trust-based\vspace{-.2cm}\\ combiner\end{tabular}};
         \node[rectangle,below of=model2,node distance=.5cm] (aid3) {};
        \node[coordinate, right of=merge, node distance=1.5cm] (end) {};
        \draw[->] (splitter) |- (model1);
        \draw[->] (splitter) |- (model2);
        \draw[->] (start) -- node[xshift=-1cm]{\begin{tabular}{c}New\vspace{-.2cm}\\regressors\end{tabular}}(splitter);
        \draw[->] (model1) -| (merge);
        \draw[->] (model2) -| (merge);
        \draw[->] (merge) -- node[xshift=1.2cm]{Predictions} (end);
        \end{tikzpicture}}
        \caption{A scheme of the structure proposed for prediction, where the expertise of gray and black box models is combined based on their suitability (here indicated as \emph{trust}) to describe the new data. For the sake of simplicity, we consider only 2 experts, but we could in principle include multiple experts of the two kinds.}\label{fig:scheme_2}
    \end{figure}
    While incorporating physical constraints into the neural local models certainly enhances the interpretability of ensemble models, the structure proposed in \cite{bischof2022mixture} still relies on experts that are all PINNs. This might make the overall structure unnecessarily complex, especially when there are known behaviors of the system that can be described by (simple) first principle-based equations.   
    To strike a balance between the intelligibility of simple first principle models and the predictive power of more complex ML models, in this work we propose a general ensemble scheme that allows us to:
    \begin{enumerate}
        \item train the local experts \emph{independently};
        \item address the challenge of identifying the 
        number of experts to be used at each time instant \cite{xiong2002grey}\VB{, limiting their variations over consecutive instants}.
    \end{enumerate}
     This is possible thanks to the formulation of a \emph{new fitting objective} that:
    \begin{enumerate}
        \item encompasses the possibility of having both \emph{collaborative} and/or \emph{competitive} experts, allowing local models that are specialized in describing a specific (local) behavior of the system, and blending experts' predictions whenever only one of them is not sufficient to characterize the system behavior;
        \item incorporates a penalization of \emph{abrupt variations} on the combination of experts used to explain the data at consecutive time instants, thus explicitly accounting for the ordering of data while preventing sudden changes in the exploited experts when physically incompatible with the behavior of the system.     
    \end{enumerate}
    This new formulation ultimately enables us to \textbf{easily blend gray box and black box models}, leading to an explainable mixture of experts as in \figurename{~\ref{fig:scheme_2}}, that can easily benefit from all the prior knowledge available on the system.
    In introducing the new cost function, inspired by the one introduced in \cite{Johansen2003} for the identification of Takagi–Sugeno fuzzy models, we further provide a statistical interpretation of each of its components, toward explaining the role of its hyper-parameters. Moreover, we propose an alternate approach to solve the proposed fitting problem that allows us to benefit from established estimation techniques to learn the experts, showing its effectiveness in both a numerical example and an experimental case study.
         

    The paper is organized as follows. Section~\ref{Sec:Setting} introduces the considered setting and goal, followed by a review of the modeling frameworks tightly connected with the one considered in this work in Section~\ref{Sec:relation}. The proposed cost for MoE fitting is then introduced in Section~\ref{Sec:Bayesian}, along with its statistical interpretation. An overview of the approach proposed to solve the fitting problem is given in Section~\ref{Sec:methodology}, while additional details are provided in Sections~\ref{sec:Experts}-\ref{sec:mixtures}. Our numerical and experimental results are then discussed in Section~\ref{sec:examples}. The paper is ended by some concluding remarks.
	
	\section{Setting and goal}\label{Sec:Setting}
	Let us consider a data-generating system of the form:
	\begin{equation}\label{eq:true_system}
		y(t)=f^{\mathrm{o}}(x(t))+e(t),
	\end{equation}
	with $y(t) \in \mathbb{R}$ and $x(t) \in \mathcal{X} \subseteq \mathbb{R}^{n_{x}}$ being the system's output and the associated \emph{regressor} (or \emph{feature vector}), respectively, while $e(t) \in \mathbb{R}$ is a zero-mean, random noise independent from $x(t)$, and $f^{\mathrm{o}}: \mathcal{X} \rightarrow \mathbb{R}$ is an \emph{unknown}, nonlinear function characterizing the features/output relationship. Nonetheless, suppose that some priors on the local behavior of the data-generating system are known (\emph{e.g.,} given by first principles). Additionally, let us assume that the number 
 $M \in \mathbb{N}$ 
 of 
 its possible operating regimes is also known.
	
	Given a \emph{training} sequence of feature/output pairs\footnote{The data should be \textquotedblleft exploratory enough\textquotedblright \ for all the operating regimes to guarantee an accurate feature/output function approximation.} $\{x(t),y(t)\}_{t=1}^{T}$ originated by the system in \eqref{eq:true_system}, our aim is to approximate its \emph{hidden} feature/output relationship with a \emph{convex combination of experts}, \emph{i.e.,}  
	\begin{subequations}\label{eq:mixture_of_models}
	\begin{equation}\label{eq:mixture}
		y(t) \approx \sum_{i=1}^{M} \omega_{i}(t) f_{i}(x(t);\theta_{i}),
	\end{equation}
	where $\theta_{i} \in \mathbb{R}^{n_{\theta_i}}$ are the parameters of each local model, whose dimension might vary dependently on their specific implementation. These \emph{local experts} can indeed be characterized by \emph{different} functions $f_{i}: \mathcal{X} \times \mathbb{R}^{n_{\theta_i}} \rightarrow \mathbb{R}$, whose features can be molded to embed \emph{constitutive priors} on the local behavior of the data-generating system. The weights $\omega_{i}(t)$, with $i=1,\ldots,M$, indicate the relative importance of each local model in describing the feature/output relationship, and they should satisfy the following:
	\begin{align}
		& \omega_{i}(t) \in [0,1],~~ \forall i \in \{1,2,\ldots,M\},~\forall t, \label{eq:probability_of_one}\\ 
		& \sum_{i=1}^{M}\omega_{i}(t)=1,~\forall t. \label{eq:sum_probability}
	\end{align}
	so that the mixture result in a convex combination of the local behaviors at each time step.
	\end{subequations}
	Since our goal is to fully characterize the \emph{convex combination of experts} from the training data, our goal is thus to \emph{concurrently learn} $(i)$ the parameters $(\theta_{1},\ldots,\theta_{M})$ of the local models and $(ii)$ the weights $(\omega_{1}(t),\ldots,\omega_{M}(t))$, for all $t=1,\ldots,T $.
	
	\begin{remark}[On the regressor $x(t)$]
		The combination of experts in \eqref{eq:mixture_of_models} can be used to characterize both \emph{static} and \emph{dynamic} feature/output relationships, depending on the nature of the regressor $x(t)$. Indeed, whenever $x(t)$ depends on past realizations of the outputs, the local models in \eqref{eq:mixture} are dynamic and, thus, their combination will result in a dynamical model.
	\end{remark}

        \section{Relationship with other modeling frameworks}\label{Sec:relation}

        In this section, we reinterpret \eqref{eq:mixture_of_models} to show its connection with \emph{probabilistic mixtures}, \emph{fuzzy}, and \emph{jump} models.  
        
	\textit{Mixture models.} Let us now regard the regressor $x(t)$ and the output $y(t)$ of the data-generating system \eqref{eq:true_system} at time $t$ as realizations of two random variables, namely $\mathbf{X}$ and $\mathbf{Y}$ respectively. Assume that $\mathbf{Y}$ can have $M$ possible conditional \emph{probability density functions} (PDFs) $p_{i}(\mathbf{Y}=y(t)|\mathbf{X}=x(t))$, with mean $\mu_{i}(x(t))$ and variance $\sigma_{i}^{2}(x(t))$, for $i=1,\ldots,M$. Let $\omega_{i}(t)$ denote the probability of the output at time $t$ being distributed according to the $i$-th PDF, namely
		\begin{equation}\label{eq:weight_prob}
			\omega_{i}(t) \sim Pr\left (\mathbf{Y} \sim p_{i}(\cdot)|\mathbf{X}=x(t)\right),~~i=1,\ldots,M.
		\end{equation}
		Under these assumptions, the conditional PDF of $\mathbf{Y}$ to $\mathbf{X}$ is given by the mixture
		\begin{equation}\label{eq:mixture_def}
			p(\mathbf{Y}\!\!=\!y(t)|\mathbf{X}\!=\!x(t))\!=\!\!\sum_{i=1}^{M}\omega_{i}(t)p_{i}(\mathbf{Y}\!\!=\!y(t)|\mathbf{X}\!=\!x(t)),
		\end{equation}
		according to the total probability theorem. By further assuming that the conditional mean value of the $i$-th component of the mixture $\mu_{i}(x(t))$ corresponds to the local model $f_{i}(x(t);\theta_{i})$, 
		 the combination of experts in \eqref{eq:mixture_of_models} corresponds to the mean value of the mixture \eqref{eq:mixture_def}, \emph{i.e.,}
		\begin{equation}
			\hat{y}(t;\Theta)=\mathbb{E}[y(t)|x(t)]=\sum_{i=1}^{M}\omega_{i}(t)f_{i}(x(t);\theta_i).
		\end{equation}
		Note that, when the \textquotedblleft local\textquotedblright \ conditional PDFs are Gaussian, namely
		\begin{equation}\label{eq:local_dist}
			p_{i}(\mathbf{Y}=y(t)|\mathbf{X}=x(t)) \sim \mathcal{N}\left(\mu_{i}(x(t)),\sigma_{i}^{2}(x(t))\right),
		\end{equation}
		 then the considered model corresponds to the mean of a \emph{Gaussian mixture}.
		 
		 \begin{remark}[Weights and features]
		 	According to the definition in \eqref{eq:weight_prob}, each $\omega_{i}(t)$ depends on $x(t)$, for $i=1,\ldots,M$. If the weights are not assumed to be statistically independent from the feature vector, this in turn implies that one has to learn the relationship between weights and features, in addition to identifying the local models and the weights.
		 \end{remark}

        \textit{Fuzzy models.} From a functional perspective, the considered model class \eqref{eq:mixture} exhibits a striking resemblance fuzzy models\footnote{The link between Gaussian Mixtures and fuzzy models using Gaussian membership functions and a t-norm product as the rule aggregation method is instead shown in \cite{bersini1997now}.} (see, \emph{e.g.,} \cite{Johansen2003,babuvska2003neuro,bartczuk2016new}). Indeed, the output of a fuzzy model can be characterized as
    \begin{subequations}\label{eq:TS_fuzzy}
        \begin{equation}
            \hat{y}(t)=\sum_{i=1}^{M}\phi_{i}(t)f_{i}(x(t)),
        \end{equation}
        where 
        \begin{equation}\label{eq:fuzzy_weights}
            \phi_{i}(t)=\frac{\varphi_{i}(t)}{\sum_{i=1}^{M}\varphi_{i}(t)},
        \end{equation}
        and $\varphi_{i}(t)\geq 0$ are the membership functions (see \emph{e.g.,} \cite{Fantuzzi96}) of the fuzzy sets dictated by the IF-THEN rules used to approximate a complex input/output relationship \cite{zadeh1965zadeh}.
        \end{subequations}
        From this definition, it is straightforward to show that
        \begin{equation*}
        \phi_{i}(t) \in [0,1],~~ \sum_{i=1}^{M}\phi_{i}(t)=1.
        \end{equation*}
        Therefore, by narrowing the class of weighting functions, fuzzy models can be seen as particular instances of the more general combination in \eqref{eq:mixture}, characterized by the specific weighting choice in \eqref{eq:fuzzy_weights}.

\textit{Jump models.} Let us assume that at each time instant $t$ there exists only one $i \in \{1,\ldots,M\}$ such that $\omega_{i}(t)=1$. Accordingly, based on \eqref{eq:sum_probability}, the following holds
\begin{equation}
    \omega_{j}(t)=0,~~\forall j\neq i,~j=1,\ldots,M.
\end{equation}
In this scenario, we can define new variable $s(t) \in \{1,\ldots,M\}$ embedding this feature of the model, \emph{i.e.,}
\begin{equation}\label{eq:s_constitutive}
	s(t)=i \iff \omega_{i}(t)=1,
\end{equation}
which corresponds to the \emph{latent mode} in \cite{bemporad2018fitting}. Under our assumption on the mixture weights, \eqref{eq:mixture_of_models} can thus be equivalently written as
\begin{equation}
	y(t) \approx f_{s(t)}\left(x(t);\theta_{s(t)}\right),
\end{equation}
where the dependence on $\Omega(t)$ is now replaced with that on $s(t)$. This descriptor matches the jump model considered in \cite{bemporad2018fitting}, while being more general, as the local models are not assumed to be equal. 
	\section{Learning convex combinations of experts from data}\label{Sec:Bayesian}
	Let $X=(x(1),\ldots,X(T))$ and $Y=(y(1),\ldots,y(T))$ denote the feature and output training sequences, respectively. Moreover,  let us indicate the vector of unknown parameters of the experts as
	\begin{equation}\label{eq:overall_parameters}
		\Theta=\begin{smallmatrix}
			\begin{bmatrix}
				\theta_{1}^{\top} & \ldots & \theta_{M}^{\top}
			\end{bmatrix}^{\top}
		\end{smallmatrix} \in \mathbb{R}^{n_{\Theta}}	
	\end{equation}
	let the set of \emph{feasible} weights $\Omega=\{\Omega(t)\}_{t=1}^{T}$ be defined as
	\begin{equation}\label{eq:feasibility_set}
		\mathcal{F}\!=\!\left\{\!\Omega: \Omega(t) \in [0,1]^{M},  \mathbbm{1}^{\!\top}\Omega(t) =1, t\!=\!1,\ldots,T\right\},
	\end{equation}
	with 
	\begin{equation}\label{eq:omega_t}
		\Omega(t)=\begin{smallmatrix}
			\begin{bmatrix}
				\omega_{1}(t) & \ldots & \omega_{M}(t)
			\end{bmatrix}^{\top}
		\end{smallmatrix}\!\!.
	\end{equation}
	To learn the parameters $\Theta$ and the weights $\Omega$ of the convex combination of experts from the training data, we introduce the following \emph{fitting objective}, 
	\begin{subequations}\label{eq:fitting_cost}
		\begin{equation}\label{eq:overall_cost1}
			J(X,Y;\Omega,\Theta)\!=\!\!\ell(X,Y;\Omega,\Theta)+r(\Theta)+\mathcal{L}(\Omega),
		\end{equation}
	that comprises three elements:
	\begin{itemize}
		\item[$(i)$] a \emph{loss function} $\ell\!: \mathbb{R} \times \mathbb{X} \times \mathcal{F} \times \mathbb{R}^{n_{\Theta}} \!\rightarrow\! \mathbb{R} \cup \{+\infty\}$;
		\item[$(ii)$] a \emph{group regularizer} $r(\Theta)$ defined as:  
		\begin{equation}
			r(\Theta)=\lambda_{\theta} \sum_{i=1}^{M}r_i(\theta_{i}),
		\end{equation}
		where $r_{i}: \mathcal{R}^{n_{\theta_{i}}} \rightarrow \mathbb{R} \cup \{+\infty\}$ is a local regularization term, potentially different for all $i=1,\ldots,M$ sub-models of \eqref{eq:mixture_of_models}, and $\lambda_{\theta}>0$ is a hyper-parameter of the fitting cost;
		\item[$(iii)$] a \emph{weight shaper} $\mathcal{L}: \mathcal{F} \rightarrow \mathbb{R} \cup \{+\infty\}$.
	\end{itemize}
\end{subequations}
	Accordingly, the fitting problem to learn the convex combination of experts \eqref{eq:mixture_of_models} can be cast as follows:
	\begin{subequations}\label{eq:optimization_and_pred}
		\begin{equation}\label{eq:optimization_problem}
			 \underset{\Theta,\Omega \in \mathcal{F}}{\mbox{min}}~J(X,Y;\Omega,\Theta),
		\end{equation}
	and it has to be concurrently solved with respect to $\Theta$ and $\Omega$. The optimal combination for a given training set $(X,Y)$ is thus defined as
	\begin{equation}\label{eq:optimal_mixture}
		(\Theta^{\star},\Omega^{\star}) \leftarrow \underset{\Omega \in \mathcal{F},\Theta}{\mathrm{argmin}}~J(X,Y;\Omega,\Theta),
	\end{equation}
	resulting in the following prediction given by the ensemble model:
	\begin{equation}\label{eq:ext_mixture}
			\hat{y}(t) \leftarrow \underset{y}{\mathrm{argmin}}~~\ell(x(t),y;\Omega^{\star}(t),\Theta^{\star}),
	\end{equation}
	whose definition stems from the separability of the fitting objective.
	\end{subequations}
	Note that, the fitting loss \eqref{eq:fitting_cost} is tightly linked with the one proposed in \cite{bemporad2018fitting} to learn \emph{jump models}, as formalized in the following lemma. 

	\begin{lemma}[Relation with jump models fitting]\label{lemma:jump_models}
		Assume that the weights $\omega_{i}(t)$ in \eqref{eq:mixture_of_models} are restricted to lay in $\{0,1\}$, for all $i=1,\ldots,M$ and $t=1,\ldots,T$. Then, the loss in \eqref{eq:fitting_cost} corresponds to the one introduced for jump models fitting in \cite{bemporad2018fitting}.
	\end{lemma}
	\begin{proof}
		The proof can be found in Appendix~\ref{appendix:A}. \hfill $\blacksquare$
	\end{proof}
	This result thus indicates that the framework proposed in this paper encompasses the one presented in \cite{bemporad2018fitting}, with jump models becoming a particular instance of the broader class of convex combination of experts \eqref{eq:mixture_of_models} considered in this work. Meanwhile, the relationship highlighted in Lemma~\ref{lemma:jump_models} allows us to provide a probabilistic interpretation to the objective function in \eqref{eq:fitting_cost}, under similar assumptions to those considered in \cite[Section 3]{bemporad2018fitting}. Indeed, when the subsequent assumptions are satisfied:
	\begin{itemize}
		\item[A1.] the weights $\mathbf{\Omega}$, the model parameters $\mathbf{\Theta}$ and the features $\mathbf{X}$ are statistically independent, \emph{i.e.,}
		\begin{align*}
			& p(\mathbf{\Omega}\!=\!\Omega|\mathbf{\Theta}\!=\!\Theta,\mathbf{X}\!=\!X)=p(\Omega|\Theta,X)=p(\Omega),\\
			& p(\mathbf{\Theta}\!=\!\Theta|\mathbf{\Omega}\!=\!\Omega,\mathbf{X}\!=\!X)=p(\Theta|\Omega,X)=p(\Theta);
		\end{align*} 
		\item[A2.] the priors on the local experts are all equal, \emph{i.e.,}
		\begin{equation*}
			p(\theta_{i})=p(\theta_{j})=p(\theta),~~\forall i,j=1,\ldots,M,
		\end{equation*}
		and the local experts are statistically independent, namely
		\begin{equation*}
			p(\Theta)=\prod_{i=1}^{M} p(\theta_{i});
		\end{equation*}
	\end{itemize}
	we can interpret the objective \eqref{eq:fitting_cost} as follows.
	\begin{proposition}[A probabilistic view on \eqref{eq:fitting_cost}]\label{prop:1}
		Let assumptions A.1-A.2 hold and define
		\begin{subequations}\label{eq:single_def1}
			\begin{align}
				&\ell(X,\!Y\!;\Omega,\!\Theta)\!=\!-\!\log p(\mathbf{Y}\!\!=\!\!Y|\mathbf{\Omega}\!=\!\Omega,\!\mathbf{\Theta}\!=\!\Theta,\!\mathbf{X}\!\!=\!\!X),\\
				&r(\theta_i)=-\log p(\theta_i),~~i=1,\ldots,M,\\
				&\mathcal{L}(\Omega)=-\log p(\Omega).
			\end{align}
		\end{subequations} 
		Then, minimizing $J(X,Y;\Omega,\Theta)$ in \eqref{eq:fitting_cost} and \eqref{eq:single_def1} with respect to $\Theta$ and $\Omega$ corresponds to the maximization of the joint likelihood $p(Y,\Omega,\Theta|X)$.
	\end{proposition}
	\begin{proof}
		By applying Bayes' theorem, we have:
		\begin{align}\label{eq:bayes1}
			\nonumber p(Y,\Omega,\Theta|X)&=p(Y,\Omega|\Theta,X)p(\Theta|X)\\
			\nonumber &=p(Y|\Omega,\Theta,X)p(\Omega|\Theta,X)p(\Theta|X)\\
			&=p(Y|\Omega,\Theta,X)p(\Omega)p(\Theta),
		\end{align}
		where the last equality is consequent to A.1. By considering its logarithm, it then results that
		\begin{equation*}
			\log p(Y,\!\Omega,\!\Theta|X)\!=\!\log p(Y|\Omega,\!\Theta,\!X)+\log p(\Omega)+\!\sum_{i=1}^{M}\!\log p(\theta_i),
		\end{equation*}
		from which our claim straightforwardly follows according to A.2 and \eqref{eq:single_def1}. \hfill $\blacksquare$
	\end{proof}
	
	\begin{remark}[Using quadratic regularization]
		If the regularizer is $r(\theta_{i})=\|\theta_{i}\|_{2}^{2}$, then the problem in \eqref{eq:optimization_problem} can be solved under the assumption of a Gaussian prior on $\theta_{i}$, \emph{i.e.,} $p(\theta_{i})=c_{\theta}e^{-\frac{\|\theta_{i}\|_{2}^{2}}{2\sigma_{\theta}^{2}}}$, with $\sigma_{\theta}=\sqrt{\frac{1}{2\lambda_{\theta}}}$. 
	\end{remark}
	\subsection{Learning an accurate combination of experts}
	Since our goal is to approximate the unknown feature/output relationship by leveraging the ensemble model in \eqref{eq:mixture_of_models}, a possible choice for the \emph{loss function} in \eqref{eq:overall_cost1} is 
	 \begin{equation}\label{eq:cost_mixture}
	 	 \ell(X,Y;\Omega,\Theta)=\sum_{t=1}^{T}\ell^{\mathrm{mix}}(y(t),\Omega(t)^{\top}\mathbf{f}(x(t);\Theta)),
	 \end{equation}
	where $\ell^{\mathrm{mix}}: \mathbb{R} \times [0,1]^{M} \times \mathbb{R}^{n_{\Theta}} \rightarrow \mathbb{R} \cup \{+\infty\}$ is a convex function penalizing the divergence between the observed outputs and the ones reconstructed by the convex combination of experts, \emph{e.g.,}
	\begin{equation}\label{eq:least_sq_mix}
		\ell^{\mathrm{mix}}(y(t),\Omega(t)\!^{\top}\!\mathbf{f}(x(t);\Theta))\!=\!c\|y(t)\!-\!\Omega(t)\!^{\top}\!\mathbf{f}(x(t);\Theta)\|_{2}^{2},
	\end{equation}
	and
	\begin{equation}
		\mathbf{f}(x(t);\Theta)=\begin{smallmatrix}
			\begin{bmatrix}
				f_{1}(x(t);\theta_{i}) & \ldots &  f_{1}(x(t);\theta_{M})
			\end{bmatrix}^{\top}
		\end{smallmatrix}.
	\end{equation}
	Based on Proposition~\ref{prop:1}, we can provide a statistical interpretation of the cost \eqref{eq:overall_cost1} implemented with this fitting loss, by introducing the additional assumption:
	\begin{itemize}
		\item[A.3.] the likelihood $p(\mathbf{Y}|\Omega,\Theta,\mathbf{X})$ is given by\footnote{We denote with $p(y(t)|\Omega,\Theta,\mathbf{X})=p(\mathbf{Y}=y(t)|\Omega,\Theta,\mathbf{X})$ and $p(y(t)|\Omega(t),\Theta,x(t))=p(\mathbf{Y}=y(t)|\Omega=\Omega(t),\Theta,\mathbf{X}=x(t))$.}
		\begin{align*}
			p(\mathbf{Y}|\Omega,\Theta,\mathbf{X})&=\prod_{t=1}^{T}p(y(t)|\Omega,\Theta,\mathbf{X})\\
			&=\prod_{t=1}^{T}p(y(t)|\Omega(t),\Theta,x(t)),
		\end{align*} 		
		with $p(y(t)|\Omega(t),\Theta,x(t))$ corresponding to the likelihood of the single output sample $y(t)$ given the feature vector $x(t)$, the weights $\Omega(t)$ and the experts, characterized by $\Theta$.
	\end{itemize}
	\begin{proposition}[Statistical characterization of \eqref{eq:cost_mixture}]\label{prop:2}
		Let A.1-A.3 be satisfied and let
			\begin{align}
				\nonumber &\ell^{\mathrm{mix}}(y(t),\Omega(t)^{\!\top}\mathbf{f}(x(t);\Theta))\!=\!-\log  p(y(t)|\Omega(t),\Theta,x(t))\\
				&\qquad \qquad \qquad \quad \!=\!-\log p(y(t)|\Omega(t)^{\!\top}\mathbf{f}(x(t);\Theta)).
			\end{align}
		Then, solving \eqref{eq:optimization_problem} with the loss defined in \eqref{eq:single_def1} and $\ell(X,Y;\Omega,\Theta)$ as in \eqref{eq:cost_mixture} implies maximizing the likelihood of $y(t)$ being described by the convex combination in \eqref{eq:mixture_of_models}. 
	\end{proposition}
	\begin{proof}
		Consider \eqref{eq:bayes1}, which can be further decomposed as:
		\begin{equation*}
			 p(Y,\Omega,\Theta|X)=p(Y|\Omega,\Theta,X)p(\Omega)p(\Theta)
		\end{equation*}
		according to A.3. Then, it can be easily shown that
		\begin{align*}
			\log p(Y,\Omega,\Theta|X)\!=&\!\sum_{t=1}^{T} \log p(y(t)|\Omega(t)^{\!\top}\mathbf{f}(x(t);\Theta))+\\
			&\qquad +\log p(\Omega)+\sum_{i=1}^{M}\log p(\theta_{i}),
		\end{align*}
		based on which the proof straightforwardly follows. \hfill $\blacksquare$
	\end{proof}

Note that, this conventional choice for the loss function \cite{esen2012twenty} focuses solely on promoting the accuracy of the combination of local experts, neglecting the quality of their standalone predictions. This approach may lead to local models that lack practical value by not accurately describing the realistic behavior of the data-generating system. The limitation becomes particularly notable when the experts are designed to incorporate different priors on the data-generating system.
	
	\begin{remark}[An interpretation of \eqref{eq:least_sq_mix}]
		When the loss function $\ell(X,Y;\Omega,\Theta)$ in \eqref{eq:overall_cost1} is set to \eqref{eq:least_sq_mix}, this implies assuming a probabilistic model for the output $y(t) \sim \mathcal{N}\left(\Omega(t)^{\!\top}\mathbf{f}(x(t);\Theta),\sigma_{y}^{2}\right)$, with $\sigma_{y}=\sqrt{\frac{1}{2c}}$.
	\end{remark}
	
	\subsection{Learning trustworthy experts}
	By focusing on learning local experts that reliably describe the behavior of the true system at different operating regimes, one can consider a \emph{loss function} \eqref{eq:overall_cost1} defined as:
	 \begin{equation}\label{eq:cost_single}
		 \ell(X,Y;\Omega,\Theta)\!=\sum_{t=1}^{T}\sum_{i=1}^{M}\omega_{i}(t)\ell_i^{\mathrm{loc}}(x(t),y(t);\theta_{i}),
	\end{equation}
	where $\ell_{i}^{\mathrm{loc}}:\mathbb{R} \times \mathbb{R}^{n_{\theta_{i}}} \rightarrow \mathbb{R} \cup \{+\infty\}$, with $i=1,\ldots,M$, are (possibly different) convex functions weighting the fitting of local experts to the training data, \emph{e.g.,}
	\begin{equation}\label{eq:least_sq_loc}
		\ell^{\mathrm{loc}}_i(x(t),y(t);\theta_{i})=c_{i}\|y(t)-f_{i}(x(t);\theta_i)\|_{2}^{2},
	\end{equation}
	when considering a least squares local loss. To provide a statistical interpretation of this loss let us drop Assumption~A.3 and suppose that 
	\begin{itemize}
		\item[A.4.] the likelihood $p(Y|\Theta,\Omega,X)$ is equal to
		\begin{align*}
			p(Y|\Theta,\Omega,X)&=p(Y|\Theta,X);
		\end{align*}
		\item[A.5.] the likelihood $p(Y|\Omega,\Theta,X)$ satisfies 
			\begin{equation*}
			p(Y|\Theta,X)=\prod_{t=1}^{T}\prod_{i=1}^{M}p(y(t)|\theta_i,x(t)),
		\end{equation*} 		
		where $p(y(t)|\theta_{i},x(t))$ is the likelihood of $y(t)$, given the feature $x(t)$ and the $i$-th expert.
	\end{itemize}
	We can then provide the following probabilistic interpretation of \eqref{eq:cost_single}.
	\begin{proposition}[Statistical characterization of \eqref{eq:cost_single}]\label{prop:3}
		Assume that A.1-A.2 and A.4-A.5 are satisfied and define the local loss $\ell_{i}^{\mathrm{loc}}(x(t),y(t);\theta_{i})$ as the weighted log-likelihood
		\begin{equation}\label{eq:assum_local}
			\ell_{i}^{\mathrm{loc}}(x(t),y(t);\theta_{i})=-\frac{1}{\omega_{i}}\log p(y(t)|\theta_i,x(t)).
		\end{equation}
		Then, minimizing $J(X,Y;\Omega,\Theta)$ given by \eqref{eq:overall_cost1} and \eqref{eq:cost_single} is equivalent to maximizing the joint probability $p(Y,\Omega,\Theta|X)$, looking at the weighted likelihood of $y(t)$ being described by each local expert.
	\end{proposition}	
	\begin{proof}
		Based on A.4, \eqref{eq:bayes1} is equal to
		\begin{equation*}
			p(Y,\Omega,\Theta|X)=p(Y|\Theta,X)p(\Omega)p(\Theta),
		\end{equation*}
		whose logarithm corresponds to
		\begin{align*}
			&\log p(Y,\Omega,\Theta|X)\!=\log p(\Omega)+\\
			&\qquad \qquad \quad +\!\!\sum_{t=1}^{T}\!\left[\sum_{i=1}^{M}\log p(y(t)|\theta_i,\!x(t))\!+\!\log p(\theta_i)\right]\!\!,
		\end{align*}
		according to A.5. Our claim straightforwardly follows from the definition in \eqref{eq:assum_local}.
		\hfill $\blacksquare$
	\end{proof}
 While this loss function, already introduced in \cite{esen2012twenty}, explicitly considers the quality of local models, it overlooks the accuracy of their combination on its own. Nonetheless, this combination can be crucial for approximating the true data-generating system when the blending of different experts is needed. Indeed, relying solely on this loss function may hinder the ensemble model in \eqref{eq:mixture_of_models} from generating trustworthy predictions of the overall system's behavior.
	
	\begin{remark}[An interpretation of \eqref{eq:least_sq_loc}]
		Let $\ell(X,Y;\Omega,\Theta)$ be set as in \eqref{eq:cost_single}, with the local losses chosen as \eqref{eq:least_sq_mix}. Then, for each of the experts, we assume a Gaussian probabilistic model for the output, \emph{i.e.,} $y(t) \sim \mathcal{N}\left(f_{i}(x(t);\theta_i),\sigma_{y,i}^{2}\right)$, with $\sigma_{y,i}=\sqrt{\frac{1}{2c_{i}}}$. The trustworthiness of this assumption is then accounted for via the weight $\Omega(t)$.
	\end{remark}
 
	\subsection{Blending experts fidelity and ensemble accuracy}
	For the learned convex combination of experts \eqref{eq:mixture_of_models} to be accurate and the sub-models to be valid standalone descriptors of the local behavior of \eqref{eq:true_system}, the loss functions in \eqref{eq:cost_mixture} and \eqref{eq:cost_single} can be combined into:
 \begin{subequations}\label{eq:models_fit}
	\begin{align}
	    \ell(X,Y;\Omega,\Theta)=\sum_{t=1}^{T} \ell_{t}(x(t),y(t);\Omega(t),\Theta),
	\end{align}	
        where
        \begin{align}
        \nonumber &\ell_{t}(x(t),y(t);\Omega(t),\Theta)=\ell^{\mathrm{mix}}(y(t),\Omega(t)^{\top}\mathbf{f}(x(t);\Theta))\\
        &\qquad \qquad  +\beta \sum_{i=1}^{M}\omega_{i}(t)\ell_i^{\mathrm{loc}}(x(t),y(t);\theta_{i}),
        \end{align}
        \end{subequations}
	and $\beta>0$ is a tunable weight, trading off between the \emph{cooperativeness} and the \emph{competitiveness} of the local experts. Indeed, the first term promotes their collaboration towards an accurate combination, while the second term steers them to \emph{compete} towards being standalone descriptors of the system's behavior at each instant. According to the probabilistic interpretations provided in Proposition~\ref{prop:2}-\ref{prop:3}, $\beta$ can be interpreted as a measure of the relative \textquotedblleft trustworthiness\textquotedblright of assumptions A.3 and A.5. Accordingly, by exploiting the loss in \eqref{eq:models_fit}, the learning problem becomes \emph{multi-objective}.
	
	\subsection{Shaping the combination of experts}
	The weight shaper $\mathcal{L}(\Omega)$ can be selected to enhance the interpretability of the combination of experts and to \emph{explicitly} account for the \emph{temporal ordering} of data. A possible approach to make this possible is to introduce a shaper that discourages abrupt changes in the weights over consecutive time steps, thus enforcing changes in the \textquotedblleft active\textquotedblright \ experts to be as smooth as possible.\\  
    To 
    achieve this goal, the shaper can specifically be designed as follows:
	\begin{subequations}
	\begin{equation}\label{eq:alphatrans_cost}
		\mathcal{L}(\Omega)=\sum_{t=2}^{T}\mathcal{L}^{\mathrm{trans}}(\Omega(t),\Omega(t-1))
	\end{equation}		
	where 
	\begin{equation}
		\mathcal{L}^{\mathrm{trans}}(\Omega(t),\Omega(t-1))=\eta\|\Omega(t) -\Omega(t-1) \|_{2}^{2},
	\end{equation}
	\end{subequations}
	that tends to steer consecutive vectors of weights to be similar, according to the tunable penalty $\eta>0$. With an eye on practical applications, this specific shaper accounts for the fact that unceasing switches from one operating regime to another are seldom compliant with a real physical system. Meanwhile, we can provide a probabilistic interpretation of this specific shaper based on 
	 the Markov property holds  for $\Omega$, namely
	\begin{itemize}
		\item[A.6.] the probability of $\Omega(t)$ given the past weights $\Omega(1)$, $\Omega(2),\ldots,\Omega(t-1)$ is $p(\Omega(t)|\Omega(t-1))$.
	\end{itemize}
	and on the fact that
	\begin{itemize}
		\item[A.7.] the conditioned probability $p(\Omega(t)|\Omega(t-1))$ satisfies
		\begin{equation*}
			p(\Omega(t)|\Omega(t-1))=\mathcal{N}(\Omega(t-1),\sigma_{\Omega}^{2}I),~~\forall t=2,\ldots,T,
		\end{equation*}
		with $\sigma_{\Omega}^{2}=\sqrt{\frac{1}{2\tau}}$;
		\item[A.8.] the initial set of weights $\Omega(t)$ has a uniform probability, \emph{i.e.,} $p(\Omega(1))=I$.
	\end{itemize}
	
	\begin{proposition}[Statistical view on \eqref{eq:alphatrans_cost}]
		Under A.1-A.2, minimizing $J(X,Y;\Omega,\Theta)$ in \eqref{eq:overall_cost1} with the weight shaper in \eqref{eq:alphatrans_cost} corresponds to the maximization of $p(Y,\Omega,\Theta|X)$ based on assumptions A.8-A.10.
	\end{proposition}
	\begin{proof}
		According to A.8-A.10, the following holds
		\begin{align*}
			p(\Omega)&=p(\Omega(1),\Omega(2),\ldots,\Omega(T))=\\
			&=\prod_{t=2}^{T}p(\Omega(t)|\Omega(t-1),\ldots,\Omega(1))p(\Omega(1))\\
			&=\prod_{t=2}^{T}p(\Omega(t)|\Omega(t-1))=c_{\Omega}\prod_{t=2}^{T}e^{-\frac{\|\Omega(t)-\Omega(t-1)\|_{2}^{2}}{2\sigma_{\Omega}^{2}}},
		\end{align*}
		whose logarithm is 
		\begin{equation*}
			\log p(\Omega)=-\eta \sum_{t=2}^{T} \|\Omega(t-1)-\Omega(t)\|_{2}^{2}+\log c_\Omega.
		\end{equation*}
		 Since $c_{\Omega}=(2\pi)^{-M/2}\sigma_{\Omega}^{-1/2}$, as thus the last term is negligible when solving \eqref{eq:optimization_problem} with $\mathcal{L}(\Omega)$ defined as in \eqref{eq:alphatrans_cost}, then the proof straightforwardly follows from \eqref{eq:bayes1}.\hfill $\blacksquare$
	\end{proof}
	
	

	\section{Fitting experts combinations: an alternated approach}\label{Sec:methodology}		
		 \begin{algorithm}[!tb]
		\caption{Fitting a convex combination of experts}
		\label{algo1}
		~~\textbf{Input}: training data $X$, $Y$; maximum number of experts $M$; initial weights $\Omega^{1} \in \mathcal{F}$; parameters $\beta, \lambda_{\theta},\tau \geq 0$.
		\vspace*{.1cm}
		\hrule\vspace*{.1cm}
		\begin{enumerate}[label=\arabic*., ref=\theenumi{}]
			\item \textbf{for} $k=1,2,\ldots$ \textbf{do}
			\begin{enumerate}[label=\theenumi{}.\arabic*, ref=\theenumi{}.\arabic*]
				\item\label{step:model_fitting} $\Theta^{k+1} \leftarrow \underset{\Theta}{\mathrm{argmin}}~\ell(X,Y;\Omega^{k},\Theta)+r(\Theta);$
				\item \label{step:weight_fitting} $\Omega^{k+1} \leftarrow \underset{\Omega \in  \mathcal{F}}{\mathrm{argmin}}~\ell(X,Y;\Omega,\Theta^{k+1})+\mathcal{L}(\Omega);$
			\end{enumerate}
			\item \textbf{until termination}
		\end{enumerate}
		\vspace*{.1cm}\hrule\vspace*{.1cm}
		~~\textbf{Output}: estimated experts $\Theta^{\star}=\Theta^{k+1}$ and weights $\Omega^{\star}=\Omega^{k+1}$.
	\end{algorithm}	
	The approach we propose is to jointly learn the experts' parameters $\Theta$ and the weights $\Omega \in \mathcal{F}$ from a given set of training data $X,~Y$, ultimately solving \eqref{eq:optimization_problem}, is the coordinate descent method summarized in Algorithm~\ref{algo1}. By alternatively keeping one of the two optimization variables fixed, the proposed approach allows one to learn the local experts by solving an \emph{unconstrained} optimization problem (see step~\ref{step:model_fitting}), with the feasible weight set $\mathcal{F}$ constraining the optimization at step~\ref{step:weight_fitting} of Algorithm~\ref{algo1} only. By following the same steps of \cite{bemporad2022}, we are capable of providing some results on the convergence of the proposed approach \VB{when the losses are non-decreasing and both them and the experts are convex}. To this end, let us introduce the following definition \cite{bemporad2022}.
	\begin{definition}[Partial optimal solution for \eqref{eq:optimization_problem}]\label{def1}
		A pair $(\Theta^{\star},\Omega^{\star}) \in \mathbb{R}^{n_{\Theta}} \times \mathcal{F}$ is a partial optimal solution for \eqref{eq:optimization_problem} given that it satisfies the following
		\begin{align}
			& J(X,Y;\Omega^{\star},\Theta^{\star}) \leq J(X,Y;\Omega,\Theta^{\star}), \forall \Omega \in \mathcal{F},\\
			&J(X,Y;\Omega^{\star},\Theta^{\star}) \leq J(X,Y;\Omega^{\star},\Theta), \forall \Theta \in \mathbb{R}^{n_{\Theta}}.
		\end{align}
	\end{definition}
	Moreover, let us assume the following.
	\begin{assumption}[On the successive solutions of step~\ref{step:weight_fitting}]\label{assu:1}
		For $\Theta^{k+1}=\Theta^{k}$ there exists only one weight sequence $\Omega^{k+1} \in \mathcal{F}$ such that
		\begin{equation*}
			\Omega^{k+1} \leftarrow \arg\min_{\Omega \in \mathcal{F}}~~J(X,Y;\Omega,\Theta^{k+1}),
		\end{equation*}
		and that sequence satisfies $\Omega^{k+1}=\Omega^{k}$. 
	\end{assumption}
	Note that this assumption entails that only one set of weights can be associated to a given set of experts' parameters and, thus, that only one of their combinations leads to a minimal loss when they are fixed. Accordingly, it holds that exists a unique minimum of $J(X,Y;\Theta,\Omega)$ for fixed $\Theta$. 
	We can now formalize the following result on the convergence of Algorithm~\ref{algo1}. 
	\begin{theorem}[On the convergence of Algorithm~\ref{algo1}]
		Let the local experts in \eqref{eq:mixture_of_models} be all convex. Moreover, let the loss function $\ell(X,Y;\Omega,\Theta)$, the group regularizer $r(\Theta)$ and the weight shaper $\mathcal{L}(\Omega)$ be convex, non-negative and \VB{non-decreasing} functions of the data and the optimization variables. Additionally, suppose that the optimization problem to be solved at step~\ref{step:weight_fitting} of Algorithm~\ref{algo1} is feasible at each iteration. Then, under Assumption~\ref{assu:1}, Algorithm~\ref{algo1} converges to a partial optimal solution of \eqref{eq:optimization_problem} in a finite number of steps.
	\end{theorem}
	\begin{proof}
		The proof follows the one of \cite[Theorem 1]{bemporad2022} and it is aimed at showing that Algorithm~\ref{algo1} is a convergent block-coordinate descent algorithm \cite{beck2013convergent}. Given the fixed weights $\Omega=\Omega^{k}$, \eqref{eq:optimization_problem} becomes 
		\begin{equation*}
			\min_{\Theta} \ell(X;Y,\Omega^{k},\Theta)+r(\Theta),
		\end{equation*}
		 which is a convex, unconstrained optimization problem in $\Theta$, since the cost is a combination of convex functions. Therefore, this problem can be solved up to global optimality and 
 	 	\begin{equation*}
		 	\Theta^{k+1} \leftarrow \arg\min_{\Theta}  \ell(X;Y,\Omega^{k},\Theta)+r(\Theta),
		 \end{equation*}
		 satisfies $J(X,Y;\Omega^{k},\Theta^{k+1}) \!\leq\! J(X,Y;\Omega^{k},\Theta)$, $\forall \Theta \in \mathbb{R}^{n_{\Theta}}$. Analogously, when $\Theta=\Theta^{k+1}$, then  \eqref{eq:optimization_problem} becomes 
		 \begin{equation*}
		 	\min_{\Omega \in \mathcal{F}}~~J(X,Y;\Omega^{k},\Theta),
		 \end{equation*}  
	 	which is a convex optimization problem with convex constraints in $\Omega$, which is assumed to be always feasible and solved to global optimality. As a consequence, the updated sequence
	 	\begin{equation*}
	 		\Omega^{k+1} \leftarrow \arg\min_{\Omega \in \mathcal{F}}~~J(X,Y;\Omega,\Theta^{k+1}),
	 	\end{equation*}
 		satisfies
 		\begin{subequations}
 		\begin{align}
 			& J(X,Y;\Omega^{k+1},\Theta^{k+1}) \leq J(X,Y;\Omega^{k},\Theta^{k+1}),\\
 			& J(X,Y;\Omega^{k},\Theta^{k+1}) \leq J(X,Y;\Omega^{k},\Theta^{k}), 
 		\end{align}
 		\end{subequations}
 		In turn, this implies that $J(X,Y;\Omega,\Theta)$ is monotonically non-increasing at each iteration of Algorithm~\ref{algo1}, while the cost is lower bounded by zero since all its component are non-negative. Therefore, the sequence of optimal costs converges asymptotically. Moreover, thanks to Assumption~\ref{assu:1}, it also converges in a finite number of steps, since no chattering between different weights is possible, to a point that is a partial optimum of \eqref{eq:optimization_problem} according to Definition~\ref{def1} since all steps are solved to global optimality.   
	\end{proof}
	
	Note that the previous convergence result holds even of the order of the iterations is reversed. Therefore, this proof holds even if   we initialize the local experts before the weights. Moreover, this convergence result provides indications on possible termination criteria for Algorithm~\ref{algo1}. In particular, $(i)$ by introducing two (eventually equal) tolerances $\epsilon_{\Theta},\epsilon_{\Omega}>0$, Algorithm~\ref{algo1} can be stopped when
	\begin{subequations}\label{eq:termination_1}
	\begin{align}
		& \|\Theta^{k+1}-\Theta^{k}\|_{2}<\epsilon_{\Theta}, \\
		& \|\Omega^{k+1}-\Omega^{k}\|_{2}<\epsilon_{\Omega},
	\end{align}  
	\end{subequations}
	Alternatively, $(ii)$ by setting a tolerance $\epsilon_{J}>0$, the algorithm can be stopped when
	\begin{equation}\label{eq:termination_2}
		|J(X,Y;\Omega^{k+1},\Theta^{k+1})-J(X,Y;\Omega^{k},\Theta^{k})|<\epsilon_{J}.
	\end{equation}
	or it can be terminated $(iii)$ by fixing a maximum number $k_{\mathrm{max}}$ of iterations.\\	
The provided convergence result for Algorithm 1 does not ensure the attainment of a global minimum due to its sensitivity to the chosen weight initialization. It is thus recommended to explore multiple initial sequences $\Omega^{1,(i)}$ (for $i=1,\ldots,N$), run Algorithm 1 for each, and select the combination of experts yielding the smallest cost. In the presence of priors on the experts' operating regimes, $\Omega^{1}$ can be chosen accordingly, thus incorporating this prior into learning. Otherwise, the weights can be initially set randomly. 

	\section{Experts learning}\label{sec:Experts}
	Let us now focus on the optimization problem to be solved at step~\ref{step:model_fitting} of Algorithm~\ref{algo1}. When the loss function $\ell(X,Y;\Omega,\Theta)$ is chosen as in \eqref{eq:cost_single}, the fitting problem to be solved at step~\ref{step:model_fitting} becomes
	\begin{subequations}\label{eq:experts_fitting_problem}
	\begin{equation}
		\min_{\Theta} \sum_{i=1}^{M}\!\left[\sum_{t=1}^{T}\omega_{i}^{k}(t)\ell_i^{\mathrm{loc}}(x(t),\!y(t);\theta_{i})\!+\!\lambda_{\theta}r(\theta_{i})\right]\!,
	\end{equation}  
	that is separable across the experts. Therefore, local models' fitting can be carried out in parallel, by solving $M$ distinct optimization problems
	\begin{equation}\label{eq:loc_minimization}
		\min_{\theta_{i}}~\sum_{t=1}^{T}\omega_{i}^{k}(t)\ell_i^{\mathrm{loc}}(x(t),y(t);\theta_{i})\!+\!\lambda_{\theta}r(\theta_{i}),
	\end{equation}  
 	\end{subequations}
	with $i=1,\ldots,M$. When the local losses $\ell_{i}^{\mathrm{loc}}$, $i=1,\ldots,M$, and $r$ are convex functions, these optimization problems can be solved globally (up to the wanted precision) with standard convex programming tools \cite{BoydConvex}.    \\	
	Meanwhile, when either the loss defined in \eqref{eq:cost_mixture} or the one in \eqref{eq:models_fit} are exploited, the separability of the fitting problem is lost due to the features of $\ell^{\mathrm{mix}}(y(t),\Omega(t)^{\top}\mathbf{f}(x(t);\Theta))$. Specifically, 
	when the more general loss \eqref{eq:models_fit} is considered, step~\ref{step:model_fitting} of Algorithm~\ref{algo1} requires the solution of the \emph{sharing problem}
	\begin{subequations}\label{eq:sharing_formulation}
		\begin{equation}\label{eq:sharing_pb}
			\min_{\Theta}~\sum_{i=1}^{M} h_{i}(\hat{y}_{i}(\theta_i))+\sum_{t=1}^{T}g\!\left(\sum_{i=1}^{M}\omega_{i}^{k}(t)\hat{y}_{i}(t;\theta_i)\right)\!,
		\end{equation} 		
		where
		\begin{equation}\label{eq:local_estimator}
			\hat{y}_{i}(t;\theta_i)=f_{i}(x(t);\theta_{i}),
		\end{equation}
		denotes the output at time $t$ predicted by the $i$-th local expert based on $x(t)$, and
		\begin{align}
			 & h_{i}(\hat{y}_{i}(\theta_i))\!=\!\beta \sum_{t=1}^{T}\omega_{i}^{k}(t)\ell_{i}^{\mathrm{loc}}(x(t),y(t);\theta_{i})\!+\!r(\theta_{i}),\label{eq:individual}\\
			 & g\!\left(\sum_{i=1}^{M}\omega_{i}^{k}(t)\hat{y}_{i}(t;\theta_i)\!\right)\!\!=\!\ell^{\mathrm{mix}}(y(t),\!\Omega(t)\!^{\top}\mathbf{f}(x(t);\Theta)) \label{eq:shared}.
		\end{align}
	\end{subequations}
	By looking at \eqref{eq:shared}, it is clear that this term of the cost is not separable over the experts, since it is used to characterize the fitting objective that one aims at achieving by exploiting them as an ensemble. This implies that the experts cannot be trained individually and, thus, that several local models of (potentially) different complexity and nature have to be learned at once.\\	
	To overcome this issue, we propose to exploit the \emph{Alternating Direction Method of Multipliers} (ADMM) \cite{boydADMM,GABAY1983}. To this end, we introduce a set of auxiliary variables $z(t) \in \mathbb{R}^{M}$, for $t=1,\ldots,T$, allowing us to recast the problem as
	\begin{subequations}\label{eq:sharing_ADMM}
	\begin{align}
		&\min_{\Theta,Z}~\sum_{i=1}^{M} h_{i}(\hat{y}_{i}(\theta_i))+\sum_{t=1}^{T}g\!\left(\sum_{i=1}^{M}z_{i}(t)\right)\\
		\nonumber & ~~ \mbox{s.t. }z_{i}(t)\!=\!\omega_{i}^{k}(t)\hat{y}_{i}(t;\theta_i),~i\!=\!1,\ldots,M,\\
		& \qquad \qquad \qquad \qquad  \qquad ~\quad t\!=\!1,\ldots, T. \label{eq:equalities_sharing}
	\end{align}
	\end{subequations}
	Let us consider the associated augmented Lagrangian:
	\begin{align}
		\nonumber & L(\Omega^{k},\Theta,Z,U)\!=\!\!\sum_{i=1}^{M} h_{i}(\hat{y}_{i}(\theta_i))\!+\!\sum_{t=1}^{T}g\!\left(\sum_{i=1}^{M}z_{i}(t)\right)\!+\!\\
		&\qquad \quad \quad \quad ~~~~~ +\frac{\rho}{2}\sum_{t=1}^{T}\sum_{i=1}^{M}\delta_i^2(\theta_i,z_i(t),u_{i}(t)), \label{eq:Lagrangian}
	\end{align}
	with $\rho>0$ being a tunable parameter, $Z\!=\!(z(1),\ldots,z(T))$,  $u_{i}(t)$ being the (scaled) Lagrange multipliers associated with the equality constraints in \eqref{eq:equalities_sharing}, for $i=1,\ldots,M$ and $t=1,\ldots,T$ and
	\begin{equation}\label{eq:deltai_1}
		\delta_i(\theta_i,z_i(t),u_{i}(t))=\omega_{i}^{k}(t)\hat{y}_{i}(t;\theta_i)\!-\!z_{i}(t)\!+\!u_{i}(t).
	\end{equation}
	The scaled form of ADMM for problem \eqref{eq:sharing_ADMM} is:
	\begin{subequations}
	 	\begin{align}
		& \theta_{i}^{j+1} \!\!\leftarrow\! \underset{\theta_i}{\mathrm{argmin}}~ h_{i}(\hat{y}_{i}(\theta_i)\!)\!+\!\!\frac{\rho}{2}\!\sum_{t=1}^{T}\!\delta_i^2(\theta_i,\!z_i^j(t),\!u_{i}^j(t)), \label{eq:ADMM1_1}\\
		& Z^{j+1} \!\!\!\leftarrow\! \underset{Z}{\mathrm{argmin}}~L(\Omega^{k}\!,\!\Theta^{j\!+\!1}\!,Z,\!U^{j}),\label{eq:ADMM1_2}\\
		& u_{i}^{j+\!1}(t)\leftarrow u_{i}^{j}(t)\!+\!\omega_{i}^{k}(t)\hat{y}_{i}(t;\theta_i^{j+1})\!-\!z_{i}^{j+1}(t),\label{eq:ADMM1_3}
	\end{align}
	where $j=1,\ldots$ is the ADMM iteration counter, \eqref{eq:ADMM1_1} has to be carried out for all $i=1,\ldots,M$, the Lagrange multipliers have to be updated as in \eqref{eq:ADMM1_3} for all $i=1,\ldots,M$ and $t=1,\ldots,T$, and the auxiliary variables are computed via the minimization of
	\begin{equation}
		L(\Omega^{k}\!,\!\Theta^{j\!+\!1}\!,Z,\!U^{j})\!=\!\! \sum_{t=1}^{T}\!\! \left[g\!\left(\!\sum_{i=1}^{M}\!\!z_{i}(t)\!\right)\!\!+\!\frac{\rho}{2}\Delta(\Theta^{j\!+\!1}\!,Z,U^{j})\right]\!\!,
	\end{equation}
	with
	\begin{equation}
		\Delta(\theta^{j+1},Z,U^{j})=\sum_{i=1}^{M}\delta_i^2(\theta_i^{j+1},z_i(t),u_{i}^j(t)).
	\end{equation}
	\end{subequations}
	As shown in  \eqref{eq:ADMM1_1}, this approach allows for the parallelization of the experts training, thus enabling the use of tools for the optimization of the local parameters $\theta_{i}$ customized to the nature of each expert. Nonetheless, the ADMM formulation in \eqref{eq:sharing_ADMM} involves the introduction of $TM$ additional optimization variables, ultimately increasing the complexity of the learning problem to be solved. 

	An approach to overcome this limitation, involves the introduction of the \emph{weighted average}	\begin{equation}\label{eq:weighter_avgAux}
		\bar{z}(t)=\frac{1}{M}\sum_{i=1}^{M}\omega_{i}^{k}(t)z_{i}(t),
	\end{equation} 
	in place of the local auxiliary variables $\{z_{i}(t)\}_{i=1}^{M}$, for $t=1,\ldots,T$, and the corresponding Lagrange multipliers
	\begin{equation}
		\bar{u}(t)=\frac{1}{M}\sum_{i=1}^{M} u_{i}(t),~t=1,\ldots,T.
	\end{equation}
	Accordingly, by defining $\bar{Z}=(\bar{z}(1),\ldots,\bar{z}(T))$, the ADMM steps to solve \eqref{eq:sharing_formulation} become
	\begin{subequations}
		\begin{align}
			& \theta_{i}^{j+1} \!\leftarrow\!\underset{\theta_i}{\mathrm{argmin}}~ h_{i}(\hat{y}_{i}(\theta_i))\!+\frac{\rho}{2}\sum_{t=1}^{T}[\tilde{\delta}_{i}^{j}(t;\theta_i)]^{2},  \label{eq:ADMM2_1}\\
			& \bar{Z}^{j+1} \!\!\leftarrow\! \underset{\bar{Z}}{\mathrm{argmin}}\sum_{t=1}^{T}\!\left(\!g(M\bar{z}(t))\!+\!\frac{\rho M}{2}[\bar{\varepsilon}^{j}(\bar{z}(t))]^{2}\right)\!,\label{eq:ADMM2_2}\\
			& \bar{u}^{j+\!1}(t)\!\leftarrow \bar{u}^{j}(t)\!+\!\frac{\Omega^{k}(t)\mathbf{f}(x(t);\Theta^{j+1})}{M}\!-\!\bar{z}^{j+1}(t),\label{eq:ADMM2_3}
		\end{align}
	where
	\begin{equation}
		 \tilde{\delta}_{i}^{j}(t;\theta_i)\!=\!\omega_{i}^{k}(t)e_{i}^{j}(t)\!+\!\frac{\Omega^{k}(t)\mathbf{f}(x(t);\Theta^{j})}{M}\!-\!\bar{z}^{j}(t)\!+\!\bar{u}^{j}(t),
		\end{equation}
	with $e_{i}^{j}(t)=\hat{y}_{i}(t;\theta_i)-\hat{y}_{i}(t;\theta_i^{j})$, and
	\begin{equation}
		 \bar{\varepsilon}^{j}(\bar{z}(t))=\bar{z}(t)\!-\!\bar{u}^{j}(t)\!-\!\frac{1}{M}\Omega^{k}(t)\mathbf{f}(x(t);\Theta^{j+1}),
	\end{equation}
	\end{subequations}
	The reader is referred to Appendix~\ref{Appendix:B} for details on their derivation.


	
	\section{Learning \& predicting the mixture weights}\label{sec:mixtures}

    \begin{figure*}[!tb]
    \centering
    \begin{tikzpicture}
        \draw[draw=blue!5!white,fill=blue!5!white] (-.25,.5) rectangle (3.5,-.5); 
        \draw[draw=blue!10!white,fill=blue!10!white] (2.5,.5) rectangle (6.5,-.5); 
        \draw[draw=blue!20!white,fill=blue!15!white] (5.5,.5) rectangle (9.5,-.5); 
        \draw[draw=blue!30!white,fill=blue!30!white] (10.5,.5) rectangle (14.25,-.5);
        \draw[draw=blue!40!white,dashed] (-.25,.5) rectangle (3.5,-.5); 
        \draw[draw=blue!50!white,dashed] (2.5,.5) rectangle (6.5,-.5); 
        \draw[draw=blue!60!white,dashed] (5.5,.5) rectangle (9.5,-.5); 
        \draw[draw=blue!70!white,dashed] (8.5,.5) rectangle (9.75,-.5); 
        \draw[draw=white] (9.75,.5) rectangle (9.75,-.5); 
        \draw[draw=blue!80!white,dashed] (10.25,.5) rectangle (11.5,-.5);
        \draw[draw=white,dashed] (10.25,.5) rectangle (10.25,-.5);
        \draw[draw=blue!90!white,dashed] (10.5,.5) rectangle (14.25,-.5);
        
        \node[draw,circle,fill=blue!2!white] (node1W1) {};
        \node[draw,circle,fill=blue!2!white,right of=node1W1,node distance=1cm] (node2W1) {};
        \node[draw,circle,fill=blue!2!white,right of=node2W1,node distance=1cm] (node3W1) {};
        \node[draw,circle,fill=blue!2!white,pattern=north east lines,right of=node3W1,node distance=1cm] (node4W1) {};
        \node[draw,circle,fill=blue!2!white,right of=node4W1,node distance=1cm] (node2W2) {};
        \node[draw,circle,fill=blue!2!white,right of=node2W2,node distance=1cm] (node3W2) {};
        \node[draw,circle,fill=blue!2!white,pattern=north east lines,right of=node3W2,node distance=1cm] (node4W2) {};
        \node[draw,circle,fill=blue!2!white,right of=node4W2,node distance=1cm] (node2W3) {};
        \node[draw,circle,fill=blue!2!white,right of=node2W3,node distance=1cm] (node3W3) {};
        \node[draw,circle,fill=blue!2!white,pattern=north east lines,right of=node3W3,node distance=1cm] (node4W3) {};
        \node[coordinate,right of=node4W3,node distance=.5cm] (conj1) {};
        \node[coordinate,right of=conj1,node distance=1cm] (conj2) {};
        \node[draw,circle,fill=blue!2!white,pattern=north east lines,right of=conj2,node distance=.5cm] (node1Wend) {};
        \node[draw,circle,fill=blue!2!white,right of=node1Wend,node distance=1cm] (node2Wend) {};
        \node[draw,circle,fill=blue!2!white,right of=node2Wend,node distance=1cm] (node3Wend) {};
        \node[draw,circle,fill=blue!2!white,right of=node3Wend,node distance=1cm] (node4Wend) {};
        \node[coordinate,right of=node4Wend,node distance=.75cm] (endnode) {};

        \draw[-] (node1W1) -- (node2W1);
        \draw[-] (node2W1) -- (node3W1);
        \draw[-] (node3W1) -- (node4W1);
        \draw[-] (node4W1) -- (node2W2);
        \draw[-] (node2W2) -- (node3W2);
        \draw[-] (node3W2) -- (node4W2);
        \draw[-] (node4W2) -- (node2W3);
        \draw[-] (node2W3) -- (node3W3);
        \draw[-] (node3W3) -- (node4W3);
        \draw[-] (node4W3) -- (conj1);
        \draw[-,dashed] (conj1) -- (conj2);
        \draw[-] (conj2) -- (node1Wend);
        \draw[-] (node1Wend) -- (node2Wend);
        \draw[-] (node2Wend) -- (node3Wend);
        \draw[-] (node3Wend) -- (node4Wend);
        \draw[->] (node4Wend) -- node[near end,xshift=1cm]{samples}(endnode);

    \end{tikzpicture}
	\caption{Learning mixture weights: windowing strategy for $W=4$. Successive windows are depicted as rectangles of increasingly deep shades of blue, samples associated with two consecutive windows are filled with a line pattern and associated with the window dictating their weights.}
	\label{fig:windowing}
    \end{figure*}
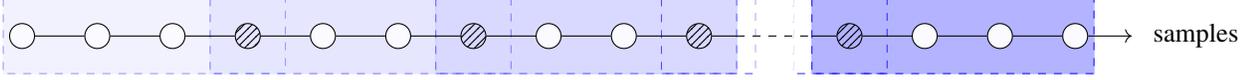

      Once the parameters of the local models have been estimated, we have to determine the weights associated with each expert for each sample of the training set. To this end, let us consider the problem to be solved at step~\ref{step:weight_fitting} of Algorithm~\ref{algo1}, \emph{i.e.,}
	\begin{equation}
		\min_{\Omega \in \mathcal{F}} \ell(X,Y;\Omega,\Theta^{k+1})+\mathcal{L}(\Omega),
	\end{equation}
	which is a convex problem in the weights. This problem is completely separable over time when $\mathcal{L}(\Omega)$ does not penalize abrupt transitions in the temporal trends of the weights, \emph{i.e.,}
	\begin{equation}
		\mathcal{L}^{\mathrm{trans}}(\Omega(t),\Omega(t-1))=0,
	\end{equation}   
 Under the assumption mentioned, the optimization problem remains computationally tractable even with large datasets. 
 However, when smooth transitions of weights are needed, the objective function becomes non-separable, increasing computational intensity, particularly with large datasets. To address this, a windowing strategy is proposed to enable weight computation in such scenarios.
 
    Instead of considering all data at once, we use successive windows of length $W$ with a single step overlap (see \figurename{~\ref{fig:windowing}} for a schematic of the windowing strategy), and learn the weights by solving
    \begin{subequations}
		\begin{equation}
			\min_{\Omega \in \mathcal{F}} \ell^{W}(X,Y;\Omega,\Theta^{k+1})+\mathcal{L}^{W}(\Omega),
		\end{equation}
	where
	\begin{align}
		&\ell^{W\!\!}(X,Y;\Omega,\Theta^{k+1})\!=\!\!\!\!\!\!\sum_{\iota=\tau}^{\tau+W-1}\!\!\!\ell_{\iota}(y(\iota),x(\iota);\Omega(\iota),\Theta^{k+1}),\\
		&\mathcal{L}^{W}(\Omega)=
  \eta\!\!\!\!\sum_{\iota=\tau+2}^{\tau+W-1}\!\!\!\mathcal{L}^{\mathrm{trans}}(\Omega(\iota),\Omega(\iota-\!1)),
	\end{align}
 \end{subequations}
        and $\tau=1,W,\ldots,T$. Note that, as shown in \figurename{~\ref{fig:windowing}}, the last weights vector estimated in each window is overwritten by the first computed for the following one.\\
    Despite its sub-optimality, this approach scales even to large datasets and enforces smoothness in the evolution of the weights. At the same time, it introduces an additional hyper-parameter, \emph{i.e.,} the length of the window $W$, whose fine-tuning is crucial to obtain reliable results. Indeed, short windows reduce the computational time required to estimate the weights over the training set, at the price of possibly resulting in poor estimates. On the other hand, the accuracy of the estimated weights is likely to increase when $W$ increases, but so does the computational load and the training time. 
    
    Once Algorithm~\ref{algo1} has been iterated and the convex combination of experts is ready for deployment, it is fundamental to have a procedure that allows for the prediction of the convex combination's weights to ultimately forecast the behavior of the system. Indeed, given an unseen regressor sequence $x_{v}(t)$, $t=1,\ldots,T_{v}$ we must indicate a set of estimates $\{\hat{\Omega}_{v}(t)\}_{t=1}^{T_{v}}$ for the associated weights to predict the outputs as
    \begin{equation}
        \hat{y}_{v}(t)=\sum_{i=1}^{M}\hat{\omega}_{v}(t)f_{i}(x_{v}(t);\theta^{\star}_{i}),~~t=1,\ldots,T_{v}
    \end{equation}
    where $\Theta^{\star}$ results from Algorithm~\ref{algo1}. 

    Focusing on one-step-ahead predictions, $\hat{\Omega}_{v}(t)$ can be obtained by following the same rationale exploited for inference in \cite[Section 4.2]{bemporad2018fitting}. Namely, one can either solve
    \begin{subequations}
    \begin{equation}
        (\hat{y}_{v}(t),\hat{\Omega}_{v|t}) \!\leftarrow \arg\!\min_{y,\Omega} J_{t}(X_{v|t},Y_{v|t-1},y;\Omega,\Theta^{\star\!}),
    \end{equation}
    where $\hat{\Omega}_{v|t}\!=\!\{\hat{\Omega}_{v}(\tau)\}_{\tau=1}^{t}$ are the weights to be estimated up to time $t$, $X_{v|t}\!=\!\{x_{v}(\tau)\}_{\tau=1}^{t}$, $Y_{v|t-1}=\{y_{v}(\tau)\}_{\tau=1}^{t-1}$ and
    \begin{align}
        \nonumber & J_{t}(X_{v|t},Y_{v|t-1},y;\Omega,\Theta^{\star})=\ell_{t}(y,x_{v}(t);\Omega(t)\Theta^{\star})\\
        &~~~~~+\sum_{\tau=1}^{t-1}\ell_{\tau}(y_{v}(\tau),x_{v}(\tau);\Omega(\tau),\Theta^{\star})\!+\!\mathcal{L}_{t}(\Omega),
    \end{align}
    with
    \begin{equation*}
        \mathcal{L}_{t}(\Omega)=
        \eta \!\sum_{\tau=2}^{t}\!\mathcal{L}^{\mathrm{trans}}(\Omega(\tau),\Omega(\tau-\!1)),
    \end{equation*}
    \end{subequations}
     for all $t=1,\ldots,T_v$. Alternatively, when recursive inference is considered, one can iteratively solve
    \begin{equation}
   (\hat{y}_{v}(t),\hat{\Omega}_{v}(t)) \!\leftarrow \!\arg\!\!\min_{y,\Omega(t)}~J_{t|t-1}(y,x_{v}(t);\Omega(t),\Theta^{\star}),
    \end{equation}
    by choosing
    \begin{align*}
        \nonumber J_{t|t-1}(y,x_{v}(t);\Omega(t),\Theta^{\star})=&
        \ell_{t}(y,x_{v}(t);\Omega(t);\Theta^{\star})\\
        &\qquad +
        \mathcal{L}_{t|t-1}(\Omega(t)),
    \end{align*}
    and
    \begin{equation*}
    \mathcal{L}_{t|t-1}(\Omega(t))=
    \eta\mathcal{L}^{\mathrm{trans}}(\Omega(t),\hat{\Omega}_{v}(t-1)),
    \end{equation*}
    for $t=1,\ldots,T_v$.

    Nonetheless, if we assume that $\omega(t)=\omega(x(t))$, $\forall t \in \{1,\ldots,T\}$, namely that the weights dictating the level of trust of each expert are dictated by the feature vector, we can exploit the weight sequence $\Omega^{\star}$ returned by Algorithm~\ref{algo1} to fit an auxiliary model (the gating) encompassing the relationship between the training features and $\Omega^{\star}$. This model devised to infer the level of confidence associated with each expert given any new regressor, provides a soft partition of the feature space that can be used to associate $\hat{\Omega}_{v}(t)$ to $x_{v}(t)$, for all $t=1,\ldots,T_{v}$. To accomplish this additional modeling task, one can leverage a pool of available architectures, depending on the application at hand and the desired level of explainability of this weight predictor. 

    \section{Case studies}\label{sec:examples}
    We now evaluate the effectiveness of the proposed structure and training strategy in two case studies, respectively used to empirically analyze the properties of the learned mixture and to highlight its effectiveness in a real-world scenario. In both cases, the performance of the learned convex combination of experts is quantitatively evaluated by looking at the \emph{mean absolute error} (MAE) and the \emph{goodness of fit} (GoF), \emph{i.e.,}:
    \begin{subequations}
    \begin{align}
    & \mathrm{MAE}=\frac{1}{T_{v}}\sum_{t=1}^{T_{V}}|y_{v}(t)-\hat{y}_{v}(t)|,\label{eq:MAE}\\
    & \mathrm{GoF}=\max\left\{1-\frac{\sum_{t=1}^{T_v}(y_{v}(t)-\hat{y}_{v}(t))^{2}}{\sum_{t=1}^{T_v}(y_{v}(t)-\bar{y}_{v})^{2}},0\right\},\label{eq:GoF}
    \end{align}
    \end{subequations}
    where $T_{v}$ is the length of the dataset against which the performance of the convex combination is assessed, $y_{v}(t)$ are the outputs to be reconstructed, $\hat{y}_{v}(t)$ are their estimates obtained with the learned combination of experts, while $\bar{y}_{v}$ is the samples average of $\hat{y}_{v}(t)$. Note that the MAE index allows us to detect inaccurate reconstructions that follow the trend of the data but are scaled in range, but its values might not be affected by relatively small oscillations around the true output. The latter nonetheless impacts the GoF index, thus making these two indicators complementary in assessing the performance of the combination of experts.   
    \subsection{Numerical example}
    \begin{table}[!tb]
		\centering
    		\caption{Numerical example: local models coefficients}\label{tab:2LocalTheta}
		\begin{tabular}{c|cccc}
		  $i$ & \textbf{$\theta_{i,1}$} & \textbf{$\theta_{i,2}$} & \textbf{$\theta_{i,3}$} & \textbf{$\theta_{i,4}$}\\
			\hline
			\hline
			\textbf{1} & 0.50 & -0.30 & 0.90 & -0.80\\
			\hline
			\textbf{2} & 0.10 & 0.40 & -0.60 & -0.50\\
            \hline
		\end{tabular}
	\end{table}
    	\begin{figure}[!tb]
		\centering
	\includegraphics[scale=.5]{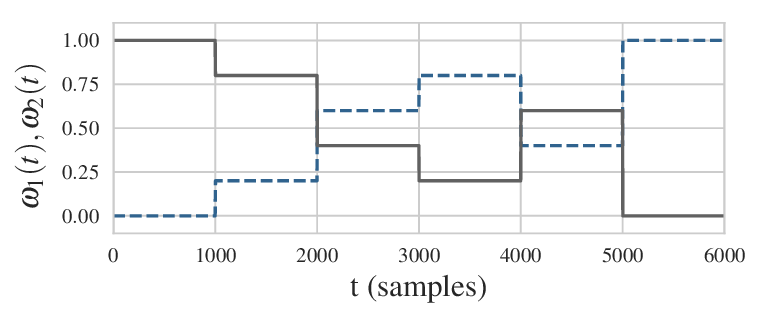}\vspace{-.2cm}
		\caption{Numerical example: true weights $\omega(t)$ used to generate the training set. The first component $\omega_{1}(t)$ is shown in black, while $\omega_{2}(t)$ is the dashed blue line.}
		\label{fig:2m_true_alpha}
	\end{figure}
    \begin{table}[!tb]
       \centering
       \caption{Numerical example: hyper-parameters of Algorithm~\ref{algo1}, 
       terminating the ADMM iterations to learn the local experts after $j_{\mathrm{max}}$ recursions.}
       \label{tab:param_toy}
       \begin{tabular}{ccccccc}
             $M$ & $\lambda_{\theta}$ 
             & $\eta$ & $\rho$ & $\beta$ & $j_{\mathrm{max}}$ & $k_{\mathrm{max}}$\\
             \hline 
             \hline
             2 & $ 5\cdot10^{-3}$ 
             & $50$ & $10^{-9}$ & $10^{-6}$& $120$ & $70$\\
            \hline
       \end{tabular}
    \end{table}
    With the aim of assessing the quality of the learned combination of experts, as a first case study we consider a data-generating system that is actually described by a convex combination of $2$ linear models, \emph{i.e.,} the true system in \eqref{eq:true_system} is characterized by
    \begin{subequations}
        \begin{equation}
             f^{\mathrm{o}}(t)=\sum_{i=1}^{M}\omega^{\mathrm{o}}(t)y_{i}(x(t)), 
        \end{equation}
        with 
        \begin{equation}
            y_{i}(x(t))\!=\!\underbrace{\begin{bmatrix}
                y(t\!-\!1) & y(t\!-\!2) & u(t\!-\!1) & u(t\!-\!2)
            \end{bmatrix}^{\!\top\!}}_{x^{\!\top\!}(t)}\theta_{i}
        \end{equation}
        and the parameters reported in \tablename{~\ref{tab:2LocalTheta}}.
    \end{subequations}
    The data used to learn and validate the convex combination of experts are collected by considering a pseudo-random binary signal (PRBS) $u(t)$ of length $T=6000$ samples and a Gaussian distributed, zero-mean, white noise $e(t)$ (see \eqref{eq:true_system}) with variance $\sigma_{e}^{2}=4\cdot10^{-2}$, yielding a signal-to-noise ratio (SNR) of $20.27$~[dB]. 
    The true weights used throughout the data collection are instead reported in \figurename{~\ref{fig:2m_true_alpha}}, showing that only at the beginning and at the end of our experiment a single local model is active (see the first and last $1000$ values of $\omega(t)$) while overall the two local models concur to give the output to the system.\\
    \begin{table}[!tb]
		\centering
		\caption{Numerical example: true \emph{vs} estimated experts' parameters.}\label{tab:2m_Theta}
       \resizebox{\columnwidth}{!}{
		\begin{tabular}{ccccc|cccc}
		\multicolumn{1}{c}{}	& \textbf{$\theta_{1,1}$} & \textbf{$\theta_{1,2}$} & \textbf{$\theta_{1,3}$} & \textbf{$\theta_{1,4}$} & \textbf{$\theta_{2,1}$} & \textbf{$\theta_{2,2}$} & \textbf{$\theta_{2,3}$} & \textbf{$\theta_{2,4}$}\\
			\hline
			\hline
			\multicolumn{1}{c|}{\textbf{True}} & 0.50 & -0.30 & 0.90 & -0.80 & 0.10 & 0.40 & -0.60 & -0.50\\
			\hline
			\multicolumn{1}{c|}{\textbf{Estimated}} &0.50 & -0.30 & 0.88 & -0.80 & 0.09 & 0.42 & -0.63 & -0.50\\
   \hline
		\end{tabular}}
	\end{table}
	\begin{figure}[!tb]
		\centering
           \begin{tabular}{c}
           \subfigure[$\omega_{1}(t)$ \emph{vs} $\hat{\omega}_{1}(t)$]{\includegraphics[scale=.6]{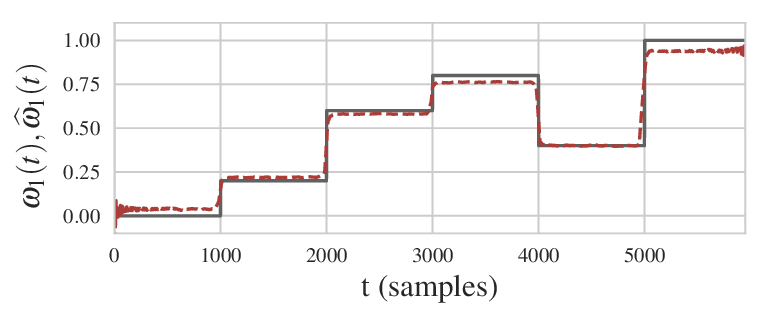}}\\
           \subfigure[$\omega_{2}(t)$ \emph{vs} $\hat{\omega}_{2}(t)$]{\includegraphics[scale=.6]{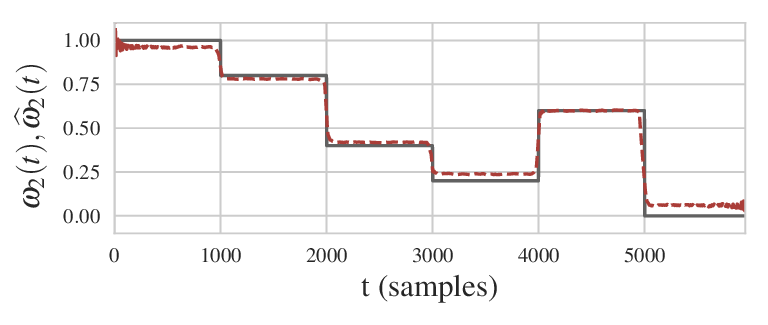}}
            \end{tabular}
		\caption{Numerical example: true (black line) \emph{vs} reconstructed weights (red dashed).}
		\label{fig:2m_Local}
	\end{figure}
    In order to assess the performance of Algorithm~\ref{algo1}, we conducted a 10-fold cross-validation procedure. This involved using 10\% of the available data for validation. We also set the hyper-parameters of Algorithm~\ref{algo1} as specified in \tablename{~\ref{tab:param_toy}}. We made the assumption that we already knew the true number of local models, $M=2$, and we considered five randomly generated possible initial sequences. The result of this training procedure is a set of estimated local experts, which are detailed in \tablename{~\ref{tab:2m_Theta}}. Remarkably, Algorithm~\ref{algo1} successfully reconstructed the true experts with only negligible errors in their parameters. Additionally, the weights converged to values very close to the true ones in less than 10 iterations, as shown in Figure \figurename{~\ref{fig:2m_Local}}.

    \subsubsection{Sensitivity analysis}
    By still carrying out a 10-fold cross-validation procedure, we now assess the sensitivity of Algorithm$~\ref{algo1}$ to the choice of its hyper-parameters and the entity of the noise corrupting the data through the quantitative indexes in \eqref{eq:MAE} and \eqref{eq:GoF}. 
    
    \textit{Sensitivity to the hyper-parameters.}
    \begin{figure}[!tb]
       \centering
       \begin{tabular}{c}\subfigure[MAE and GoF in validation \emph{vs} $\lambda_{\theta}$ \label{fig:sensitiviy_hypertheta}]{\includegraphics[scale=.6]{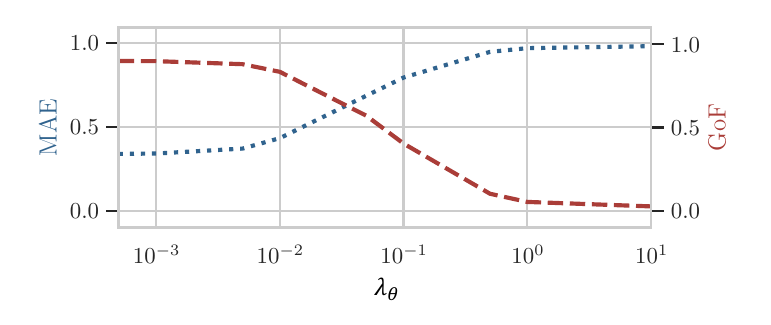}}\\
       \subfigure[MAE and GoF in validation \emph{vs} $\eta$\label{fig:sensitiviy_hypereta}]{\includegraphics[scale=.6]{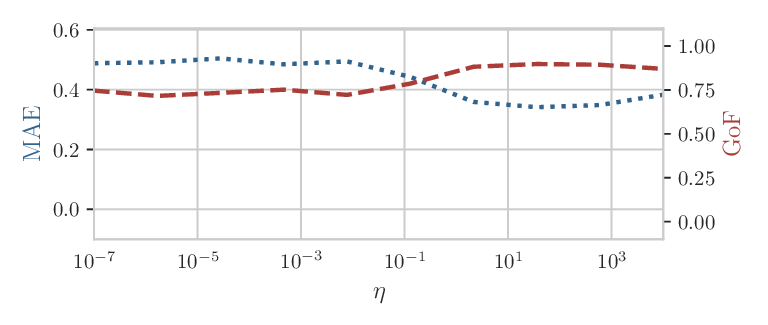}}\\
        \subfigure[MAE and GoF in validation \emph{vs} $\rho$\label{fig:sensitiviy_hyperrho}]{\includegraphics[scale=.6]{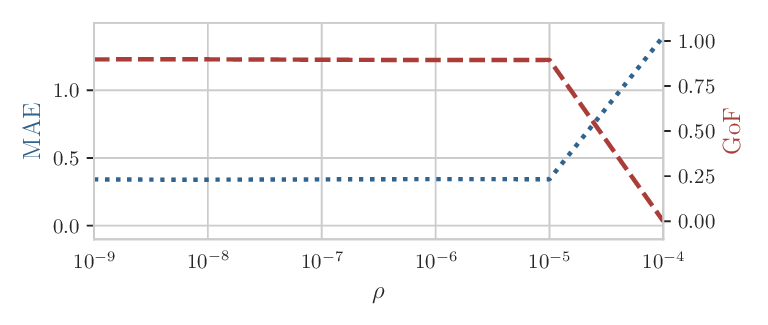}}
        \end{tabular}\vspace{-.2cm}
        \caption{Numerical example: MAE \eqref{eq:MAE} and GoF \eqref{eq:GoF} obtained in validation \emph{vs} changes in the hyper-parameters of Algorithm~\ref{algo1}.}
        \label{fig:sensitiviy_hyper}
    \end{figure}
    To analyze the sensitivity of the proposed alternated approach to the choice of the hyper-parameters, we train new convex combinations of experts by varying one hyper-parameter at a time while keeping the others fixed to their \textquotedblleft nominal\textquotedblright \ values in \tablename{~\ref{tab:param_toy}}.\\
    \VB{As expected, }
    the quality of the learned combination tends to be affected by changes in $\lambda_{\theta}$, $\eta$ and $\rho$. In particular, as summarized in \figurename{~\ref{fig:sensitiviy_hypertheta}}, the mixture converges to a configuration where 
    the 
    resemble actual one 
    when $\lambda_{\theta}$ is lower than $10^{-2}$. This result can be explained considering 
    that high normalization may 
    steer the local parameters to values that are excessively close and small.
    \\
    On the other hand, \figurename{~\ref{fig:sensitiviy_hypereta}} shows that considering excessively low values of $\eta$ tends to slightly deteriorate performance, confirming that in situations in which the level of trustworthiness of experts slowly varies over time (as in \figurename{~\ref{fig:2m_true_alpha}}) enforcing slow variations in the estimated weights positively impact on the quality of the final combination.\\
    Lastly, we observe that performance deteriorates for values of $\rho$ above $10^{-5}$ (see \figurename{~\ref{fig:sensitiviy_hyperrho}}). This result is in line with the fact that high values of $\rho$ tend to increase the influence of the augmentation term in the Lagrangian (see \eqref{eq:Lagrangian}), at the price of reducing the relative importance of the actual fitting loss one aims at optimizing.
    
    \textit{Sensitivity to noise.}
    	\begin{figure}[!tb]
		\centering
			\includegraphics[scale=.6]{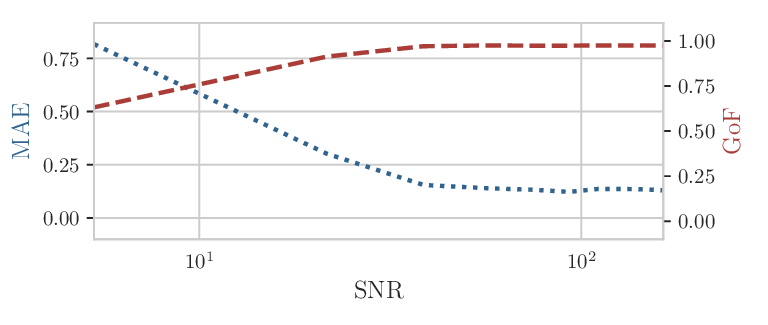}
		\caption{Numerical example: MAE \eqref{eq:MAE} and GoF \eqref{eq:GoF} in validation against noise level, measured via the signal-to-noise ratio (SNR).}\vspace{-.2cm}
		\label{fig:2s_snr}
	\end{figure}
    In evaluating the sensitivity of Algorithm~\ref{algo1}, we have generated new datasets of length $T=6000$ for different levels of noise, \emph{i.e.,} considering increasing values for the variance $\sigma_{e}^{2}$ of the noise corrupting the true combination's outputs in the interval $[10^{-6},10]$. As shown in \figurename{~\ref{fig:2s_snr}} a SNR above approximately $21$~dB already allows to attain a GoF in validation above 80\%. This is motivated by the fact that for lower values of SNR the effect of the noise makes it more difficult to identify the proper weightings, consequently deteriorating the accuracy of the experts and the quality of the reconstructed output.

    \subsection{An experimental example: side-slip estimation}
    We now evaluate the potential of the proposed structure and learning strategy by considering the problem tackled in \cite{breschi2020vehicle}, namely learning an estimator for the side-slip angle from experimental data. Specifically, the available dataset includes our target variable (namely the side-slip angle) measured via a Correvit-S-Motion non-contact optical sensor, along with the steering angle and wheel velocities, and the vehicle's linear and angular accelerations (used to construct the regressor). These data have been collected over different trials where an electric sedan has been driven on a closed, dry track in winter and summer, exploring the range\footnote{The higher values of side-slip angles corresponds to an aggressive driving style.} of side-slip angles $[0,35]$~[deg]. As also done in \cite{breschi2020vehicle}, samples associated with velocities below $20$ [km/h] have been removed, as the optimal sensor employed to measure the side-slip angle is not accurate at low velocities, and they have been down-sampled in a preliminary preprocessing phase from $200$~Hz to $25$~Hz to reduce the computational burden of the learning procedure. The data are then split into:
    \begin{itemize}
        \item a \textbf{training set}, concatenating $T_{\mathrm{w}}=6333$ samples from a winter test and $T_{\mathrm{s}}=5797$ samples referred to a summer test, for a total of $T=12130$ training data;
        \item a \textbf{validation set}, also composed of two dataset gathered over a winter and a summer test, respectively of length $5818$ and $5895$;
        \item a \textbf{test set} concatenating a winter dataset of $5771$ samples and a summer one comprising $5873$ samples, for a total of $T_{v}=11644$ data points used for testing.
    \end{itemize}

    In the following, these data are used to evaluate the performance of the convex combination when considering local experts of a different nature, \emph{i.e.,} all gray-boxes and mixtures of gray and black box models, and to compare the proposed structure and learning approach against competitors in the literature. In particular, we consider $M=2$ experts for our structure to be compatible with parallel, series (see \figurename{~\ref{fig:schemes_1}}) and the LIME+SHAP \cite{pintelas2020grey} architectures. Moreover, when training the combination of experts, we consider a single initial sequence of weights ($N=1$) dictated by the physics of the problem, namely by setting $\omega_{1}^{(0)}(t)=1$ for all $t$ such that the longitudinal acceleration $a_{x,t}\geq 0.3$~[m/s$^{2}$] and zero otherwise to distinguish between data associated with normal and aggressive driving styles. Moreover, due to the dimension of the dataset, the second step of Algorithm~\ref{algo1} is carried out by setting $\beta=10^{-11}$ and by using the windowing strategy described in Section~\ref{sec:mixtures}, by setting $W=100$. 

    When both the side-slip and the road bank angles are negligible, it is well-known that the dynamics of the side-slip angle can be inferred from the single-track model, namely \begin{subequations}\label{eq:single_track}
     \begin{align} 
     &\dot{r}_{t} \!=\! (s_{t}\!-\!y_{t})\frac{l_f c_f}{J_z} \!-\! \frac{r_{t}}{\nu_{x,t}}\left(\frac{l_f^2c_f}{J_z} \!+\! \frac{l_r^2c_r}{J_z}\right) + y_{t}\frac{l_r c_r}{J_z},
        \\
        &\dot{y}_{t}\!=\! \left(\!\frac{s_t}{\nu_{x,t}}\!-\!\frac{y_{t}}{\nu_{x,t}}\!\right)\frac{c_f}{m}\!-\! \frac{r}{\nu_{x,t}^2}\left(\!1\!-\!\frac{l_fc_f}{m} \!+\! \frac{l_rc_r}{m}\!\right) \!+\! \frac{y_t}{\nu_{x,t}}\frac{c_r}{m},
    \end{align}
    \end{subequations}
     with $y_{t}$ and $s_t$~[deg] being the side-slip and steering angles at time\footnote{The dependence on time is here indicated in the subscript, to distinguish these continuous-time variables from their discrete samples used for learning.} $t \in \mathbb{R}$, respectively, $r_t$~[deg/s] is the yaw rate of the vehicle and $\nu_{x,t}$~[m/s] is its longitudinal speed, while $c_{f},c_{r}$~[N/deg], $m$~[kg], $J_{z}$~[kg m], and $l_{f},l_{r}$~[m] are the tires' cornering stiffness, the mass of the vehicle, its moment of inertial along the vertical axes and the distances between the vehicle's center of mass and the contact point of the front and rear tires with the ground.\\     
     Despite providing a physics-informed prior on the behavior of the target variables, these equations thus have limited validity (they are accurate for \textquotedblleft small\textquotedblright \ angles only) and they are likely to provide an incomplete picture of the side-slip dynamics, especially when the driving style becomes aggressive. In this first analysis, we thus ask ourselves if the proposed convex combination of experts can improve the descriptive capabilities of the sole single-track model and which expert should pair it to attain the most accurate predictions. In particular, we consider the following three \textquotedblleft flavours\textquotedblright \ for the combination of experts.
     \begin{enumerate}
         \item \textbf{Two single-track experts} (2ST), thus considering two gray-box models, that should nonetheless be specialized in describing the behavior of the system at different operating conditions and collaborate whenever their relative simplicity does not allow them to capture the complexity of the side-slip behavior.
         \item \textbf{A single-track expert + a random forest} (ST+RF), which combines the physics informed model dictated by \eqref{eq:single_track} with a random forest with $50$ random trees, each of depth 3. In this case, the black-box expert is introduced to aid the (explainable) single-track model whenever it cannot describe the behavior of the side-slip angle. Note that, the complexity of the black-box model is still contained, to make it as interpretable as possible.
         \item \textbf{A single-track expert + a polynomial expert} (ST+P), blending a single-track model with a $3$-rd order polynomial in the regressor. This combination is introduced to still consider a black-box expert intervening when the single-track expert cannot describe the overall behavior of the system, while maintaining a level of interpretability that is superior to that of the RF expert.   
     \end{enumerate}
     All these combinations of experts have been trained by running Algorithm~\ref{algo1} and the ADMM-based routine in \eqref{eq:ADMM2_1}-\eqref{eq:ADMM2_3} for a maximum of $k_{\mathrm{max}}=20$ and $j_{\mathrm{max}}=50$ iterations, respectively, considering the regressor:
      \begin{equation}
         x(t)=\begin{bmatrix}
        \alpha_{z}(t) & \frac{r(t)}{\nu_{x}(t)} & s(t)        
         \end{bmatrix}^{\top},
     \end{equation}
     where $\alpha_{z}(t)$ is the yaw angular acceleration, also available among the measured data.

     For each configuration, the hyperparameters used for training were selected through grid-search, focusing on the penalties that yielded the best GoF on the validation set. The tests revealed that low values of $\lambda_\theta$ tend to result in overfitting, showing good performance on validation but poor generalization on the test set. Conversely, excessively high $\lambda_\theta$ values degrade performance on both validation and test sets. Therefore, $\lambda_{\theta}$ was set to $10^{2}$ to balance generalization and overfitting.
Similarly, extremely high values of $\eta$ (above $10^{5}$) were found to harm performance by excessively limiting weight variations, while lower values allowed weights to change too rapidly with data noise. Hence, $\eta$ was fixed at $10^{5}$ to strike a balance.
Performance was relatively insensitive to the choice of $\rho$ 
when other hyperparameters were well-tuned, particularly if $\rho$ remained small. Accordingly, $\rho$ was set to $10^{-5}$ for learning all combinations
.

    Meanwhile, independently from the nature of the experts, we assume the weights $\Omega$ of the convex combination of experts to be functions of the regressor, namely $\Omega(t)=\Omega(x(t))$ (see \eqref{eq:omega_t}), which is reasonable considering the modeling choices made in \cite{breschi2020vehicle} and references therein). This allows us to learn to predict the weights based on the regressors by training a regression model leveraging the outcome of Algorithm~\ref{algo1}. Here we consider a random forest regression model composed of $35$ random trees with a maximum depth of $25$.\\
    \begin{table}
        \centering
        \caption{Side-slip estimation: MAE and GoF in testing.}
        \label{tab:models_indexes}
         \begin{tabular}{cccc}
            \multicolumn{1}{c}{} & 2ST & ST+RF & ST+P\\
            \cline{2-4}
            \hline
             MAE & 1.573 & 1.601 & 1.599\\
             \hline
              GoF & 0.815 & 0.702 & 0.591 \\
              \hline
        \end{tabular}
    \end{table}
    \begin{figure*}[!tb]
        \centering
        \begin{tabular}{cc}
        \subfigure[Winter frame: true \emph{vs} reconstructed side-slip]{\includegraphics[scale=.5]{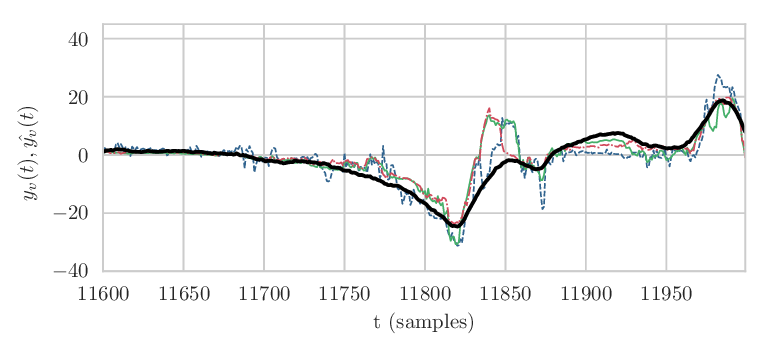}} & \subfigure[Summer frame: true \emph{vs} reconstructed side-slip]{\includegraphics[scale=.5]{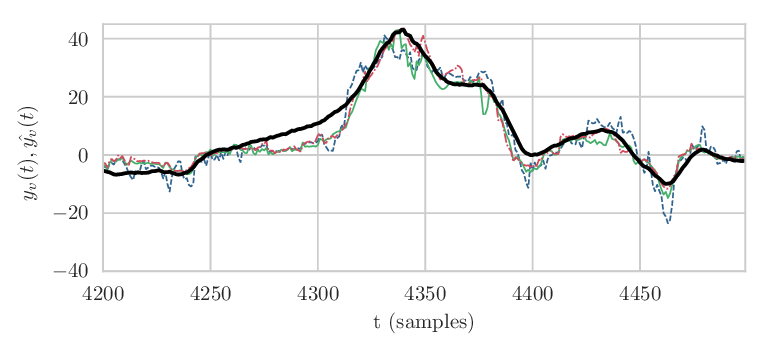}}\\
        \subfigure[Winter frame: predicted weights]{\includegraphics[scale=.5]{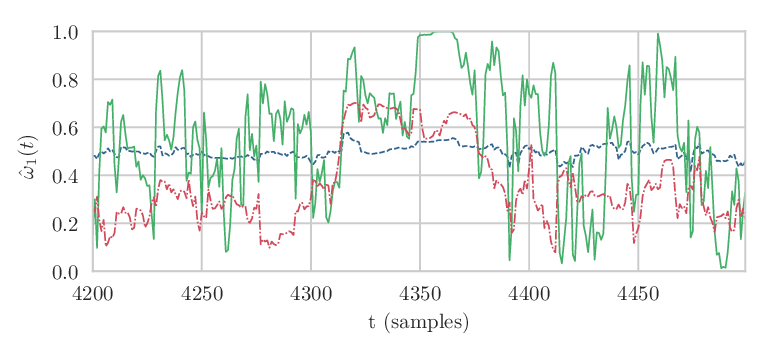}} & \subfigure[Summer frame: predicted weights]{\includegraphics[scale=.5]{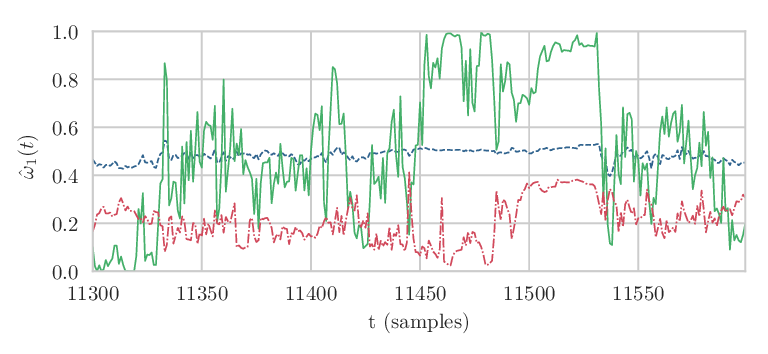}}
        \end{tabular}\vspace{-.2cm}
        \caption{Side-slip estimation: true (black) \emph{vs} predicted outputs and weights over the test set with the different architectures, 2ST (red), ST+RF (blue), and ST+P (green). For the sake of visualization, we only show the 1-st component of the weight vector over time.}\label{fig:comparison_experts}
    \end{figure*}
    The mixture of experts considered in the study accurately reconstructs side-slip behavior in various environmental conditions and driving styles, as evidenced by performance indexes and predicted trajectories. However, in the case of the ST+RF architecture, the convex combination collapses into a parallel ensemble with equal weighting, affecting interpretability. Changes in expert confidence are generally tied to changes in driving style, except for the ST+RF case, where it becomes challenging to discern which expert is better suited. Additionally, the ST+RF architecture has a tendency to predict single-point outliers in certain scenarios, as observed around the 4455-th sample.\\
    Instead, 
    the ST+P combination leads to weights that abruptly change over successive time instants which is quite realistic in practice, once again reducing the interpretability of the combination's outcome.\\
    All these observations reflect in the (slightly) better metrics achieved by the 2ST combination, as reported in \tablename{~\ref{tab:models_indexes}}, leading us not to consider any black-box model but, at least in this scenario, to stick with a combination of two gray-box ones (see \figurename{~\ref{fig:b_test}} for a comparison of the true and predictive side-slip for this structure over the entire test set).
    \begin{figure}[!tb]
 	\centering	\includegraphics[scale=.5]{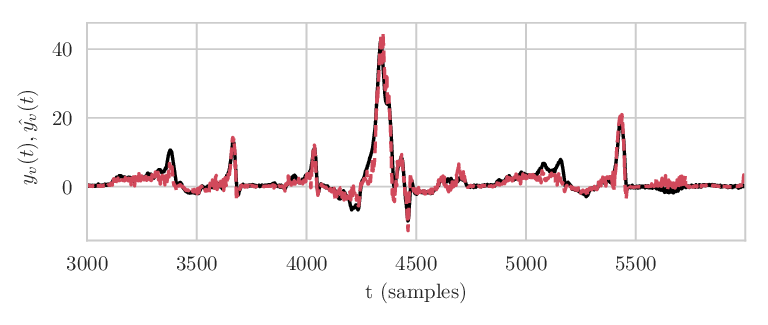}\vspace{-.2cm}
	\caption{Side-slip estimation: true (black) \emph{vs} predicted (dashed red) side-slip over the test set.}
	\label{fig:b_test}
    \end{figure}
					
    \begin{table}
    \centering
	\caption{Side-slip estimation: comparison with state-of-the-art competitors w.r.t. the performance indexes in testing.}
	\label{tab:SoAperformance}
	\begin{tabular}{ccccc}
        \multicolumn{1}{c}{} \hspace*{-.2cm}&\hspace*{-.2cm} Serial \cite{gray2018hybrid} \hspace*{-.2cm}&\hspace*{-.2cm} Parallel \cite{xiong2002grey} \hspace*{-.2cm}&\hspace*{-.2cm} LIME+SHAP \cite{pintelas2020grey} \hspace*{-.2cm}&\hspace*{-.2cm} \textbf{2ST (ours)}\hspace*{-.2cm}\\
        \cline{2-5}
        \hline
        MAE \hspace*{-.2cm}&\hspace*{-.2cm} 1.622 \hspace*{-.2cm}&\hspace*{-.2cm} 2.501 \hspace*{-.2cm}&\hspace*{-.2cm} 2.358 \hspace*{-.2cm}&\hspace*{-.2cm} \textbf{1.573}\hspace*{-.2cm}\\
        \hline
        GoF \hspace*{-.2cm}&\hspace*{-.2cm} 0.803 \hspace*{-.2cm}&\hspace*{-.2cm} 0.531 \hspace*{-.2cm}&\hspace*{-.2cm} 0.587 \hspace*{-.2cm}&\hspace*{-.2cm} \textbf{0.815}\hspace*{-.2cm}\\
        \hline 
	\end{tabular}
\end{table}
\begin{figure}[!tb]
	\centering
    \begin{tabular}{c}
    \subfigure[Winter frame: true \emph{vs} predicted side-slip]{\includegraphics[scale=.5]{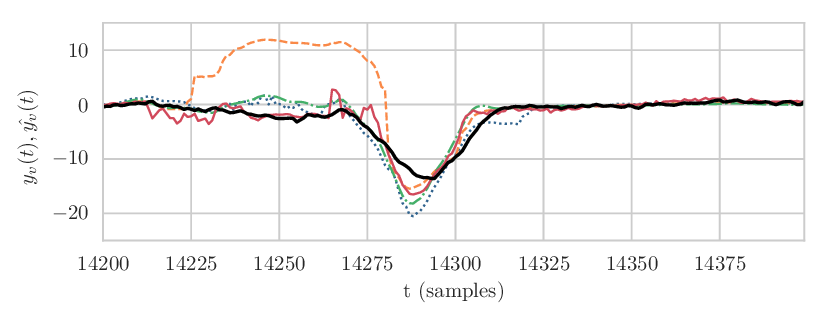}}\\
    \subfigure[Summer frame: true \emph{vs} predicted side-slip]{\includegraphics[scale=.5]{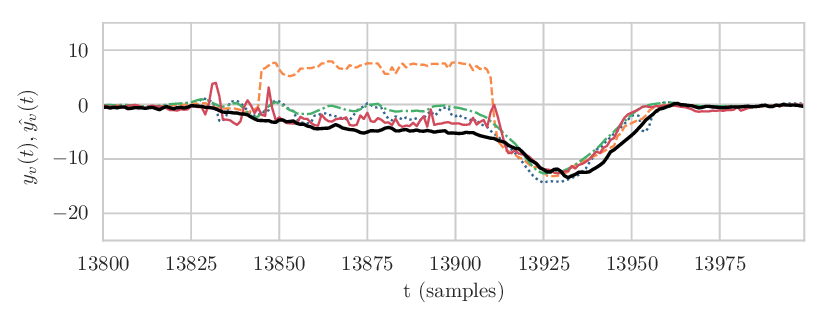}}        
    \end{tabular}\vspace{-.2cm}
	\caption{Side-slip estimation: true (black) \emph{vs} predicted side-slip over snapshots of the test set obtained with the serial (dotted blue), parallel (dashed orange), LIME+SHAP (dotted-dashed green), and our 2ST architecture (red).}
	\label{fig:b_SoAComparison}
\end{figure}
We finally conduct a comparative analysis with existing approaches, including a serial architecture from \cite{gray2018hybrid}, a parallel structure from \cite{xiong2002grey}, and the LIME+SHAP approach \cite{pintelas2020grey}. The latter is designed for large datasets, involving a two-stage process with a black-box model followed by a gray-box model.
Black-box models in all structures, including the proposed one, are selected and trained on the same dataset for the local models in the convex combination of experts. Performance indexes on the training set, detailed in \tablename{~\ref{tab:SoAperformance}}, reveal that the proposed architecture outperforms the parallel and LIME+SHAP structures, slightly surpassing the serial architecture while maintaining superior interpretability due to the inclusion of trust level information through the weights $\omega(t)$.
This superiority is emphasized by snapshots of predicted side-slip trajectories presented in \figurename{~\ref{fig:b_SoAComparison}}, demonstrating that the proposed architecture excels, especially in aggressive driving styles. It showcases improved resilience to sudden non-linear changes compared to state-of-the-art approaches, particularly evident in the predictions from the parallel structure. Even the serial approach, considered promising in blending gray and black-box models, exhibits over-shoots and under-shoots at abrupt changes, affecting its predictive performance.

\section{Conclusions}\label{Sec:conclusions}
This work revisits mixtures of experts with the aim of introducing a more flexible and interpretable structure, combining the explainability of physics-based gray models with the representation power of black-box models. The approach involves a novel cost function allowing separate training of individual experts, thus facilitating the use of standard learning tools at a local level. The objective is to achieve a balance between collaboration, competition, and explainability over the feature domain. Numerical examples and an experimental case study demonstrate the effectiveness of the proposed approach.\\
On the methodological side, future works will be devoted to analyze the convergence of each stage of the proposed learning pipeline. To enhance the method further, future efforts will focus on automating hyper-parameter tuning, recognizing its impact on the learned combination's performance.

\section*{Acknowledgments}
The work was partially supported by: Italian Ministry of Enterprises and Made in Italy in the framework of the project 4DDS (4D Drone Swarms) under grant no. F/310097/01-04/X56; FAIR (Future Artificial Intelligence Research) project, funded by the NextGenerationEU program within the PNRR-PE-AI scheme (M4C2, Investment 1.3, Line on Artificial Intelligence)”; PRIN project TECHIE: “A control and network-based approach for fostering the adoption of new technologies in the ecological transition” Cod. 2022KPHA24 CUP: D53D23001320006.\\
This study was also partially supported by the MOST – Sustainable Mobility National Research Center, Spoke 5 “Light Vehicle and Active Mobility” (M4C2, Investment 1.4 - D.D. 1033 17/06/2022, CN00000023), 

\bibliographystyle{abbrv}
\typeout{}
\bibliography{main}

\appendix
\section{Proof of Lemma~\ref{lemma:jump_models}}\label{appendix:A}
When $\Omega(t) \in \{0,1\}^{M}$ the constraint in \eqref{eq:probability_of_one} is satisfied by definition. Meanwhile, \eqref{eq:sum_probability} implies that, for each $t \in \mathbb{N}$, $\exists!~i \in \{1,\ldots,M\}~~\mbox{ s.t. } ~~\omega_{i}(t)=1,$ while $\omega_{j}(t)=0,~~\forall j \neq i, j=1,\ldots,M$.
We can thus leverage the connection between the convex combination of experts and jump models introduced in Section~\ref{Sec:relation}, and introduce the \emph{mode sequence} $S=(s(1),\ldots,s(T))$. By replacing it into the \emph{fitting cost} in \eqref{eq:overall_cost1}, the latter can be equivalently recast as:
\begin{equation}\label{eq:overall_costJump}
	J(X,\!Y\!;S,\!\Theta)\!=\!\!\sum_{t=1}^{T}\!\ell\left(x(t),\!y(t);\theta_{s(t)}\right)\!+r(\Theta)+\mathcal{L}(S),
\end{equation}
where the \emph{mode sequence loss} $\mathcal{L}(S)$ takes the place of the \emph{weight shaper} $\mathcal{L}(\Omega)$, based on \eqref{eq:s_constitutive}. This fitting cost corresponds to the one introduced in \cite[Section 2.3]{bemporad2018fitting}, thus concluding the proof.

\section{Derivation of \eqref{eq:ADMM2_1}-\eqref{eq:ADMM2_3}}\label{Appendix:B}
By exploiting \eqref{eq:weighter_avgAux}, the problem at step \eqref{eq:ADMM1_2} becomes
\begin{subequations}\label{eq:replace_aux}
	\begin{align}
		&\min_{Z,\bar{Z}}~\sum_{t=1}^{T}\left(g(M\bar{z}(t))+\frac{\rho}{2}\sum_{i=1}^{M}\|z_{i}(t)-a_{i}(t)\|_{2}^{2}\right) \label{eq:replace_auxcost}\\
		&~~  \mbox{s.t. } z(t)=\frac{1}{M}\sum_{i=1}^{M}z_{i}(t),~~t=1,\ldots,T,
	\end{align}
\begin{equation}
	a_{i}^{j+1}(t)=\omega_{i}^{k}(t)\hat{y}_{i}(t;\theta_{i}^{j+1})+u_{i}^{j}(t).
\end{equation}
\end{subequations}
By minimizing for $\bar{Z}$ fixed, the auxiliary variables can be redefined as
	$z_{i}^{j+1}(t)=a_{i}^{j+1}(t)$
that, in turn, implies 
\begin{equation*}
	\bar{z}^{j+1}(t)=\bar{a}^{j+1}(t)=\frac{1}{M}\sum_{i=1}^{M}\left(\omega_{i}^{k}(t)\hat{y}_{i}(t;\theta_{i}^{j+1})+u_{i}^{j}(t)\right).
\end{equation*}
Therefore, it holds that
\begin{equation}\label{eq:needed_substitution}
	z_{i}^{j+1}(t)=a_{i}^{j+1}(t)+\bar{z}^{j+1}(t)-\bar{a}^{j+1}(t),
\end{equation}
and, thus, the auxiliary variables can be updated solving the following unconstrained optimization problem,
\begin{equation}\label{eq:aux_update}
	\min_{\bar{Z}}~\sum_{t=1}^{T}\left(g(M\bar{z}(t))+\frac{\rho}{2}\sum_{i=1}^{M}\|\bar{z}(t)-\bar{a}^{j+1}(t)\|_{2}^{2}\right) 
\end{equation}
ultimately leading to the update in \eqref{eq:ADMM2_2}. We thus have to manipulate \eqref{eq:ADMM1_1} and \eqref{eq:ADMM1_3} for them to depend on $\bar{z}^{j}(t)$, rather than on the single auxiliary variables $\{z_{i}(t)\}_{i=1}^{M}$, with $t=1,\ldots,T$. Thanks to the equality in \eqref{eq:needed_substitution}, the Lagrange multipliers can be updated as
\begin{align}
	\nonumber u_{i}^{j+1}(t)&=\underbrace{u_{i}^{j}(t)+\omega_{i}^{k}(t)\hat{y}_{i}(t;\theta_{i}^{j+1})}_{:=a_{i}^{j+1}(t)}-z_{i}^{j+1}(t)\\
	\nonumber & = \bar{a}^{j+1}(t)-\bar{z}^{j+1}(t)\\
	&=\bar{u}^{j}(t)\!+\!\frac{\Omega^{k}(t)\mathbf{f}(x(t);\Theta^{j+1})}{M}\!-\!\bar{z}^{j+1}(t). \label{eq:equality_u}
\end{align} 
Since the right-hand side of the equality is constant with respect to the experts,  we can average both sides and derive the update in \eqref{eq:ADMM2_3}. Also, according to \eqref{eq:needed_substitution}, it holds that
\begin{align*}
	z_{i}^{j}(t)\!-\!u_{i}^{j}(t)&=z_{i}^{j}(t)\!-\!\bar{u}^{j}(t)\\
	&=a_{i}^{j}(t)+\bar{z}^{j}(t)-\bar{a}^{j}(t)\!-\!\bar{u}^{j}(t),
\end{align*}
\begin{align*}
	a_{i}^{j}(t)\!-\!\bar{a}^{j}(t)&\!=\!\omega_{i}^{k}(t)\hat{y}_{i}(t;\theta_{i}^{j})\!+\!u_{i}^{j-\!1}(t)\!-\!\frac{\hat{y}(t;\Theta^{j})}{{M}}\!-\!\bar{u}^{j-\!1}(t)\\
	&=\omega_{i}^{k}(t)\hat{y}_{i}(t;\theta_{i}^{j})-\!\frac{\hat{y}(t;\Theta^{j})}{{M}}
\end{align*}
where the last equality follows from 
	$u_{i}^{j}(t)=\bar{u}_{i}^{j}(t),~\forall j,~ i=1,\ldots,M,~t=1,\ldots,T,$
in turn, resulting from \eqref{eq:equality_u}, and 
\begin{equation}
	\hat{y}(t;\Theta^{j+1})=\Omega^{k}(t)\mathbf{f}(x(t);\Theta^{j+1}).
\end{equation}
Accordingly, by defining
$e_{i}^{j}(t)=\hat{y}_{i}(t;\theta_i)\!-\!\hat{y}_{i}(t;\theta_i^{j}),$
 we can recast \eqref{eq:deltai_1} at the $j$-th iteration as follows:
\begin{align}
	\nonumber \delta_{i}^{j}(t;\theta_i)&\!=\!\omega_{i}^{k}(t)\hat{y}_{i}(t;\theta_i)-z_{i}^{j}(t)+u_{i}^{j}(t)\\
	&\!=\!\omega_{i}^{k}(t)e_{i}^{j}(t)\!+\!\frac{\hat{y}(t;\Theta^{j})}{{M}}\!-\!\bar{z}^{j}(t)\!+\!\bar{u}^{j}(t),
\end{align} 
thus obtaining in the new update in \eqref{eq:ADMM2_1}.
\end{document}